\documentclass{article}

    \PassOptionsToPackage{numbers, compress}{natbib}

\usepackage[preprint]{neurips_2026}

\usepackage[utf8]{inputenc} 
\usepackage[T1]{fontenc}    
\usepackage{amsmath}
\usepackage{amsfonts}       
\usepackage{amssymb}
\usepackage{amsthm}
\usepackage{url}            
\usepackage{booktabs}       
\usepackage{tabularx}
\usepackage{multirow}
\usepackage{nicefrac}       
\usepackage{microtype}      
\usepackage[table]{xcolor}  
\usepackage{wrapfig}        
\usepackage{needspace}      
\usepackage{pifont}         
\usepackage{algorithm}
\usepackage{algpseudocode}
\definecolor{LinkOrange}{HTML}{D97706}
\definecolor{FeatCalRow}{HTML}{FBF3EC}
\definecolor{GainGreen}{HTML}{087F5B}
\definecolor{GainRed}{HTML}{B42318}
\definecolor{GainGray}{HTML}{6B7280}
\definecolor{GainLightGray}{HTML}{B6BDC8}
\newcommand{\posgain}[1]{\textcolor{GainGreen}{\scriptsize\,(+#1)}}

\newcommand{\upgain}[1]{\textcolor{GainGray}{\scriptsize\,\mbox{$\uparrow$#1}}}
\newcommand{\downgain}[1]{\textcolor{GainGray}{\scriptsize\,\mbox{$\downarrow$#1}}}
\newcommand{\zerogain}{\textcolor{GainLightGray}{\scriptsize\,\mbox{$\uparrow$0.0}}}
\usepackage{hyperref}       
\hypersetup{
    colorlinks=true,
    linkcolor=LinkOrange,
    citecolor=LinkOrange,
    urlcolor=LinkOrange
}
\usepackage[capitalize,nameinlink]{cleveref}
\usepackage{xspace}
\usepackage{graphicx}
\usepackage{capt-of}
\crefname{figure}{Fig.}{Figs.}
\Crefname{figure}{Fig.}{Figs.}
\crefname{table}{Tab.}{Tabs.}
\Crefname{table}{Tab.}{Tabs.}
\crefname{section}{Sec.}{Secs.}
\Crefname{section}{Sec.}{Secs.}
\crefname{appendix}{App.}{Apps.}
\Crefname{appendix}{App.}{Apps.}
\crefname{subsection}{Sec.}{Secs.}
\Crefname{subsection}{Sec.}{Secs.}
\crefname{subsubsection}{Sec.}{Secs.}
\Crefname{subsubsection}{Sec.}{Secs.}
\crefname{proposition}{Prop.}{Props.}
\Crefname{proposition}{Prop.}{Props.}
\crefname{corollary}{Cor.}{Cors.}
\Crefname{corollary}{Cor.}{Cors.}
\crefname{definition}{Def.}{Defs.}
\Crefname{definition}{Def.}{Defs.}
\crefname{remark}{Remark}{Remarks}
\Crefname{remark}{Remark}{Remarks}
\crefname{algorithm}{Alg.}{Algs.}
\Crefname{algorithm}{Alg.}{Algs.}
\algrenewcommand\algorithmicrequire{\textbf{Input:}}
\algrenewcommand\algorithmicensure{\textbf{Output:}}
\newcommand{\method}{\textsc{FeatCal}\xspace}
\newcommand{\expert}{\mathrm{exp}}
\newcommand{\merge}{\mathrm{mer}}
\newcommand{\base}{\mathrm{base}}
\newcommand{\cali}{\mathrm{cal}}
\newcommand{\tgt}{\mathrm{tgt}}
\newcommand{\anchor}{\mathrm{anc}}
\newcommand{\src}{\mathrm{src}}
\newcommand{\dep}{\mathrm{dep}}

\newtheorem{proposition}{Proposition}
\newtheorem{corollary}{Corollary}
\theoremstyle{definition}
\newtheorem{definition}{Definition}
\theoremstyle{remark}

\theoremstyle{plain}

\title{\textcolor{LinkOrange}{\textsc{FeatCal}}: Feature Calibration for Post-Merging Models}

\author{%
\textbf{Yanggan Gu}$^{1}$\thanks{Equal contribution.} \quad
\textbf{Shuo Cai}$^{1}$\footnotemark[1] \quad
\textbf{Zihao Wang}$^{2}$\footnotemark[1] \quad
\textbf{Wenjun Wang}$^{1}$\footnotemark[1] \quad
\textbf{Yuanyi Wang}$^{1}$\\
\textbf{Pengkai Wang}$^{1}$ \quad
\textbf{Sirui Huang}$^{1}$ \quad
\textbf{Su Lu}$^{1}$ \quad
\textbf{Jianmin Wu}$^{1,4}$ \quad
\textbf{Hongxia Yang}$^{1,3,4}$\thanks{Corresponding author.}\\
{\normalfont $^{1}$The Hong Kong Polytechnic University (PolyU)}\\
{\normalfont $^{2}$The Chinese University of Hong Kong}\\
{\normalfont $^{3}$Research Institute of Generative Artificial Intelligence \quad $^{4}$InfiX.ai}\\
{\normalfont \texttt{yanggangu@outlook.com}}\\
{\normalfont \textbf{Code:}
\href{https://github.com/egangu/featcal}{\underline{\texttt{github.com/egangu/featcal}}}}
}

\begin{document}

\maketitle

\begin{abstract}
Model merging combines task experts into one model and avoids joint training, retraining, or deploying many expert models, but the merged model often still underperforms task experts. We study this performance gap through \emph{feature drift}, the difference between features produced by the merged model and by the expert on the same input. Our theory decomposes this drift into upstream propagation and local mismatch, tracks how it propagates and combines through later layers in forward order, and links final feature drift to output drift. This view motivates \method, which uses a small calibration set to calibrate the merged model weights layer by layer in forward order, reducing feature drift while staying close to merged weights and preserving the benefits of model merging. \method uses an efficient closed-form solution to update model weights, with no gradient descent, iterative optimization, or extra modules. On the main CLIP and GLUE benchmarks, \method beats Surgery and ProbSurgery, the closest post-merging calibration baselines: 85.5\% vs. 77.0\%/78.8\% on CLIP-ViT-B/32 Task Arithmetic (TA) and 85.2\% vs. 83.7\%/82.2\% on FLAN-T5-base GLUE. On CLIP-ViT-B/32, 8 examples per task reach 82.9\%, and 256 examples per task take 53 seconds, about 4x faster than both baselines, showing better sample efficiency and lower calibration cost.
\end{abstract}
\vspace{-1.8em}

\begin{figure}[H]
    \centering
    \includegraphics[width=\textwidth,trim=0 4pt 0 0,clip]{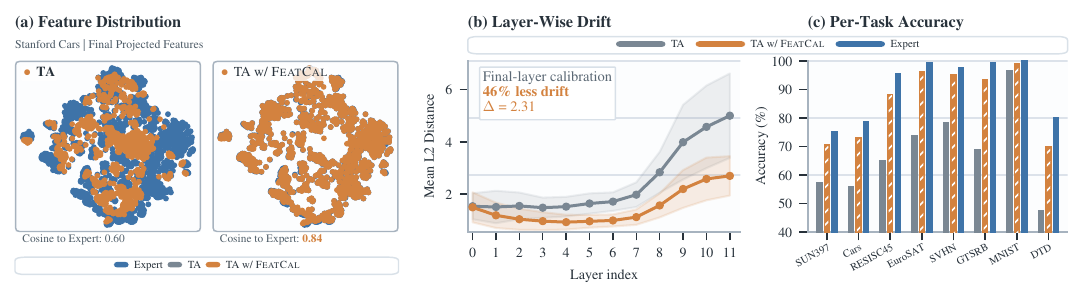}\par
    \vspace{-0.5em}
    \setlength{\abovecaptionskip}{2pt}
    \caption{
        \emph{Feature drift} after Task Arithmetic (TA) merging and \method calibration in CLIP-ViT-B/32.
        Panels (a,b) use Stanford Cars: \method moves features toward expert features, raises their mean cosine similarity from 0.60 to 0.84, and reduces mean L2 feature drift, with 46\% less final-layer drift.
        Panel (c) reports per-task accuracy in the 8-task setting. \Cref{app:ta8-feature-distribution-diagnostics} gives full 8-task feature views.
    }
    \label{fig:intro-triptych}
\end{figure}

\section{Introduction}

Model merging composes task experts into one model, avoiding joint training, retraining, or deploying a separate model for each task \citep{wortsman2022modelsoups,matena2022fisher,ilharco2023taskarithmetic,yadav2023ties,yang2024adamerging,yu2024dare,davari2024breadcrumbs,tam2024mats,daheim2024gradmatch,cheng2025wudi}. However, the merged model often still underperforms the experts it is intended to combine. This leaves a clear performance gap after merging in practice.

We analyze this gap through \emph{feature drift}, the difference between features of the merged model and the task expert on the same input sample. Our layer by layer analysis decomposes this drift into upstream propagation and local mismatch at expert input features, then shows how local mismatches propagate through later layers in \emph{forward order} and combine into final feature drift. We further analyze how final feature drift reaches the model output and becomes output drift, explaining how feature drift can affect output scores. \Cref{fig:intro-triptych} illustrates this behavior on CLIP-ViT-B/32 merged with Task Arithmetic (TA)~\citep{ilharco2023taskarithmetic}, where final features move away from the Stanford Cars expert features and drift appears across multiple layers of the network after TA merging.

Surgery and ProbSurgery~\citep{yang2024surgery,wei2025probsurgery} are related post-merging calibration methods: they identify feature drift at the final layer and train extra modules for calibration. Other intervention methods point to the same lesson: expert signals can help, but current methods often rely on task specific intervention parameters, extra modules at inference, or iterative optimization \citep{yang2024surgeryv2,osial2025intervmerge}. These choices make calibration less direct for an already merged model and can leave the final model with an auxiliary inference path. This motivates efficient, direct calibration that preserves inference speed.

\textbf{\method} uses our drift analysis as a design cue for calibrating the merged model. It uses a forward-order schedule suggested by the propagation view and calibrates the model layer by layer. We introduce a regularization term that keeps the calibrated model close to the merged model, which helps preserve the benefits of model merging and reduce overfitting to the small calibration set. The resulting objective has an efficient closed form solution, so calibration needs no gradient descent, iterative optimization, or extra modules at inference. We further introduce feature interpolation and anchor regularization to balance expert signals with the merged model and improve performance.

Empirically, \Cref{fig:intro-triptych} shows the practical effect: \method moves merged features toward expert features, reduces feature drift, and raises per-task accuracy after TA merging. On CLIP-ViT-B/32 8-task TA, it reaches 85.5\% versus 77.0\%/78.8\% for Surgery/ProbSurgery, and on FLAN-T5-base GLUE it reaches 85.2\% versus 83.7\%/82.2\%. The same trend holds on CLIP-ViT-L/14 WUDI, FLAN-T5-large, and MergeBench Llama-family LLM merging, where \method improves TA by \(+2.0\)/\(+2.3\) average points on 3B/8B models. On CLIP-ViT-B/32 TA, 8 examples per task reach 82.9\%, and calibration with 256 examples per task takes 53 seconds, about 4x faster than both baselines under the same calibration protocol.

We summarize the main contributions of this work as follows:
\par\noindent\ding{182}\enspace We develop a theory of feature drift after merging, with an exact decomposition into local mismatch and upstream propagation, forward order propagation, and a link to output drift.
\par\noindent\ding{183}\enspace We introduce \method, which efficiently calibrates merged model weights in forward order with closed form updates, without gradient descent, architecture changes, or extra modules at inference.
\par\noindent\ding{184}\enspace We validate \method on CLIP, FLAN-T5, and MergeBench LLM benchmarks, where it outperforms related post-merging calibration baselines while using fewer samples and lower calibration cost without adding inference-time modules.

\section{Related Work}

\paragraph{Model Merging.}
Most model merging methods build a fused model by merging task experts directly in parameter space \citep{wortsman2022modelsoups,matena2022fisher,ilharco2023taskarithmetic,ortizjimenez2023tangenttask,yadav2023ties,yang2024adamerging,yu2024dare,davari2024breadcrumbs,tam2024mats,daheim2024gradmatch,cheng2025wudi}. Methods based on feature statistics or feature drift, such as RegMean, RegMean++, and LOT Merging, are closer in mechanism: they derive layer updates during merging from feature statistics, regression objectives, or an explicit feature drift objective \citep{jin2023regmean,nguyen2025regmeanpp,sun2025lot}. These methods define how to build the merged model. In contrast, \method treats that model as the starting point, uses a small calibration set and task experts for post-merging calibration, and controls calibration strength through regularization. This stage separation matters because calibration must work with the feature drift left by a chosen merger instead of changing the merge rule itself. It helps preserve the benefits of model merging while reducing the risk of overfitting to calibration data.

\paragraph{Post-Merging Feature Calibration.}
Representation Surgery and follow-up methods operate on merged-model features through task-specific plugins, deeper interventions, probabilistic feature-drift modeling, or parameter-efficient modules \citep{yang2024surgery,yang2024surgeryv2,wei2025probsurgery,osial2025intervmerge}. These closest post-merging alternatives establish that expert-guided feature calibration is useful. \method differs in parameterization and deployment: instead of learning or deploying auxiliary intervention modules, it folds the calibration into the original linear module weights through closed-form regularized updates, leaving a single architecture-preserving model at inference time rather than an auxiliary intervention path.

\section{Post-Merging Feature Drift: Problem Formulation and Properties}
\label{sec:feature-drift}

Before introducing \method, we formalize post-merging feature drift and show how local mismatch is defined at expert input features, then propagated and combined in forward order across depth.
\subsection{Layer-Wise Feature Drift}
\label{sec:setup-notation}

We consider \(N\) task experts \(\{M_i^{\expert}\}_{i=1}^N\) and a merged model \(M^{\merge}\). As in standard weight-space merging, experts are fine-tuned from a common pretrained base. All models share the same architecture and contain \(L\) layers. Let \(\mathcal D_i\) be the data distribution of task \(i\). For layer \(\ell\), \(f_{i,\ell}^{\expert}\) and \(f_{\ell}^{\merge}\) denote the corresponding layer functions of the task expert and the merged model, respectively.

For an input sample \(x\) from task \(i\), define the expert and merged layer output features recursively as
\begin{equation}
\label{eq:setup-recursion}
\begin{aligned}
h_{i,0}^{\expert}(x)&=x,
&
h_{i,\ell}^{\expert}(x)&=
f_{i,\ell}^{\expert}\!(h_{i,\ell-1}^{\expert}(x)),
\\
h_{i,0}^{\merge}(x)&=x,
&
h_{i,\ell}^{\merge}(x)&=
f_{\ell}^{\merge}\!(h_{i,\ell-1}^{\merge}(x)),
\end{aligned}
\qquad \ell=1,\dots,L.
\end{equation}

\begin{definition}[Layer-wise feature drift]
\label{def:layer-wise-feature-drift}
For task \(i\), sample \(x\), and layer \(\ell\), the layer-wise feature drift is
the difference between the merged-model feature and the corresponding task-expert
feature at that layer:
\begin{equation}
\label{eq:layer-wise-feature-error}
e_{i,\ell}(x)
=
h_{i,\ell}^{\merge}(x)-h_{i,\ell}^{\expert}(x).
\end{equation}
\end{definition}
This pointwise drift signal is the object propagated in forward order in the analysis below.

\subsection{Local Mismatch and Drift Propagation}
\label{sec:layer-wise-feature-error}
\label{sec:error-propagation-nonlinear}

\begin{proposition}[Exact layer-wise drift decomposition]
\label{prop:exact-single-layer-decomposition}
For every task \(i\), sample \(x\), and intermediate layer \(\ell\), the drift decomposes as follows. By \cref{eq:setup-recursion}, \(p_{i,1}(x)=0\).
\begin{equation}
\label{eq:exact-single-layer-decomposition}
e_{i,\ell}(x)
=
\underbrace{
f_{\ell}^{\merge}\!\left(h_{i,\ell-1}^{\expert}(x)\right)
-
f_{i,\ell}^{\expert}\!\left(h_{i,\ell-1}^{\expert}(x)\right)
}_{\text{local mismatch}\ \scriptstyle m_{i,\ell}(x)}
+
\underbrace{
f_{\ell}^{\merge}\!\left(h_{i,\ell-1}^{\merge}(x)\right)
-
f_{\ell}^{\merge}\!\left(h_{i,\ell-1}^{\expert}(x)\right)
}_{\text{upstream-drift propagation}\ \scriptstyle p_{i,\ell}(x)}.
\end{equation}
\end{proposition}
\noindent\emph{Proof.} The algebraic identity and its regularity details are deferred to \cref{app:proof-feature-drift}.

\paragraph{Interpretation.}
\Cref{prop:exact-single-layer-decomposition} decomposes layer-wise feature drift into two terms: local mismatch and upstream-drift propagation. The local mismatch \(m_{i,\ell}(x)\) measures the feature mismatch caused by the difference between the merged and expert layer maps at the same expert input feature. The propagation term \(p_{i,\ell}(x)\) measures how feature drift inherited from earlier layers changes the input feature of layer \(\ell\) and is then carried into the output feature of this layer.

\begin{proposition}[Layer-wise propagation of local mismatch]
\label{prop:layerwise-propagation-identity}
Fix a task \(i\), sample \(x\), and layers \(1,\dots,L\). Suppose that, for
each layer \(\ell\), \(f_\ell^{\merge}\) is continuously differentiable on an
open neighborhood containing the segment between
\(h_{i,\ell-1}^{\expert}(x)\) and \(h_{i,\ell-1}^{\merge}(x)\). Let
\(A_{i,\ell}(x)\) denote the corresponding path-averaged local sensitivity
operator, defined explicitly in \cref{eq:avg-jacobian-general}. Then the drift
obeys the layer-by-layer recursion
\begin{align}
p_{i,\ell}(x)
&=
A_{i,\ell}(x)e_{i,\ell-1}(x),
\qquad
e_{i,\ell}(x)
=
A_{i,\ell}(x)e_{i,\ell-1}(x)+m_{i,\ell}(x).
\label{eq:layerwise-propagation-recursion}
\end{align}
Since \(e_{i,0}(x)=0\), the final feature drift is
\begin{align}
e_{i,L}(x)
&=
\sum_{\ell=1}^{L}
P_{i,\ell\rightarrow L}(x)m_{i,\ell}(x),
\notag\\
P_{i,\ell\rightarrow L}(x)
&=
A_{i,L}(x)A_{i,L-1}(x)\cdots A_{i,\ell+1}(x),
\qquad
P_{i,L\rightarrow L}(x)=I .
\label{eq:final-layer-propagation-expansion}
\end{align}
The product composes compatible local maps. See
\cref{eq:general-layerwise-unroll} for the \(s\)-to-\(t\) form.
\end{proposition}
\noindent\emph{Proof.} The local sensitivity derivation and unrolled expansion are deferred to \cref{app:proof-feature-drift}.

\paragraph{Interpretation.}
\Cref{prop:layerwise-propagation-identity} gives a simple forward-order view: local mismatches arise at individual layers, and their induced feature drift propagates through downstream layers to the final layer. Thus, final feature drift is a downstream combination of local mismatches from different layers. For residual networks, we further show that residual paths carry upstream feature drift through the skip connection. The drift can also grow under specific conditions. See \cref{app:residual-growth}.

\subsection{From Feature Drift to Output Drift}
\label{sec:output-drift}

\begin{definition}[Output drift for task scores]
\label{def:output-score-drift}
For the merged model \(M^{\merge}\), let
\(z_i^{\expert}(x)\) and \(z_i^{\merge}(x)\) be the expert and merged
task scores, each represented as a score vector. Let \(\psi_i^{\expert}\) and
\(\psi_i^{\merge}\) map final features to task scores, with
\(z_i^{\expert}(x)=\psi_i^{\expert}(h_{i,L}^{\expert}(x))\) and
\(z_i^{\merge}(x)=\psi_i^{\merge}(h_{i,L}^{\merge}(x))\). The post-merging output drift is
\begin{equation}
\label{eq:output-drift-definition}
\Delta z_i^{\merge}(x)=z_i^{\merge}(x)-z_i^{\expert}(x).
\end{equation}
\end{definition}
In the logit or similarity settings used below, \(z_i\) can contain class logits, CLIP candidate scores
(scaled similarity scores over a fixed candidate set), or fixed-prefix decoder
vocabulary logits.

\begin{proposition}[Feature-to-output perturbation]
\label{prop:main-feature-to-output-perturbation}
Suppose \(\psi_i^{\merge}\) is locally \(B_i^{\merge}(x)\)-Lipschitz on the segment between \(h_{i,L}^{\expert}(x)\) and \(h_{i,L}^{\merge}(x)\).
Here \(B_i^{\merge}(x)\) locally bounds how much the merged task score map can change when its final-feature input moves along this segment. Then
\begin{equation}
\label{eq:main-score-map-bound}
\left\|\Delta z_i^{\merge}(x)\right\|_2
\le
B_i^{\merge}(x)\left\|e_{i,L}(x)\right\|_2
+
\delta_i^{\psi,\merge}(x),
\end{equation}
where
\[
\delta_i^{\psi,\merge}(x)
=
\left\|
\psi_i^{\merge}\!\left(h_{i,L}^{\expert}(x)\right)
-
\psi_i^{\expert}\!\left(h_{i,L}^{\expert}(x)\right)
\right\|_2
\]
is the score map mismatch. If the task score map is shared, this term is \(0\).
\end{proposition}
\noindent\emph{Proof.} This is the perturbation bound proved in \cref{prop:score-map-perturbation}.

\paragraph{Interpretation.}
\Cref{prop:main-feature-to-output-perturbation} shows that final feature drift can propagate to output drift, thus further affect model outputs and task loss. For example, in a language model under a fixed prefix, token probabilities are obtained by applying softmax to output logits. Holding other logits fixed, a larger token logit gives a larger token probability, so output-logit drift can potentially induce probability drift and change the next-token choice. Cross-entropy loss is also tied to the probability assigned to the target token, so probability drift can induce loss drift. Detailed analysis is deferred to \cref{app:output-drift-bridge}.

\section{\method: Feature Calibration for Post-Merging Models}
\label{sec:method}
\label{sec:featcal}

The preceding drift analysis guides the design of a direct calibration procedure for an
already merged model. \method visits layers in forward order, takes a current
feature snapshot after earlier layer updates, and solves expert-guided
calibration objectives for the modules in that layer with explicit regularization.

\subsection{Basic Calibration Objective}

\paragraph{Forward-order calibration.}
The drift analysis in \cref{sec:error-propagation-nonlinear} shows that
feature drift arises from local mismatch terms that propagate layer by layer in
forward order. \method follows the same order by calibrating the merged model
layer by layer: after earlier layers are calibrated, it recollects current
calibrated-model features for the next layer before fitting that layer's module
objectives. This schedule also reduces mismatch between features used during
calibration and features exposed by the deployed calibrated model;
\Cref{app:forward-order-calibration} formalizes this source/deployed feature
mismatch.

\paragraph{Why calibrate linear modules?}
Under this schedule, each layer provides a shared feature snapshot. Given the
cached input features and expert target features for a module in that snapshot,
the local surrogate can be defined for any linear module. In practice, \method
applies it to the modules configured for calibration.
Linear modules are natural feature-mixing and projection points in the
architectures we study, including attention and MLP linear modules. Once their
calibrated-model input features are fixed, each linear module gives a tractable
regularized regression problem with a closed-form update. The update replaces
the existing merged weight, preserving the architecture without gradient
descent, adapters, or inference-time modules.
For other affine components, including bias parameters and LayerNorm affine
parameters, we also design calibration updates, as detailed in
Appendix~\ref{app:affine-layernorm-calibration}. In practice, these extra
updates give limited gains: on CLIP, their average accuracy improvement is less
than 0.5 percentage points. The main gains come from calibrating linear modules.

\paragraph{Linear module feature drift.}
For a fixed linear module inside the current layer, we define a module-local
version of feature drift using the variables available at that module. Let
\(W_i,W^{\merge},W^{\base}\in\mathbb{R}^{m\times d}\) be the task-\(i\)
expert, merged, and base weights for this module. When processing the layer, we
cache fixed input feature matrices \(X_i^{\cali}\) and \(X_i^{\expert}\) for this
module on the same task-\(i\) calibration samples. Here \(X_i^{\cali}\) is
produced by the prefix-calibrated model and is therefore the deployed input
source, while \(X_i^{\expert}\) is the corresponding task-expert feature matrix.

Following the layer-wise feature drift definition in
\cref{eq:layer-wise-feature-error}, we define the linear module feature drift of
a candidate calibrated weight \(W\) by
\begin{equation}
\label{eq:featcal-module-drift}
e_i^{\mathrm{linear}}(W)
=
W X_i^{\cali}
-
W_i X_i^{\expert}.
\end{equation}
As in the layer-wise analysis, this module-level drift can be interpreted as a
combination of module-local mismatch and upstream-drift propagation.
\paragraph{Basic calibration objective.}
With input features fixed, the per-module calibration objective minimizes the
overall linear module feature-drift error with a merged-weight penalty:
\begin{equation}
\label{eq:featcal-basic}
W^\star
=
\arg\min_{W \in \mathbb{R}^{m \times d}}
\sum_{i=1}^N
\| W X_i^{\cali} - W_i X_i^{\expert} \|_F^2
\;+\;
\lambda_{\merge}
\| W - W^{\merge} \|_F^2.
\end{equation}
The quadratic penalty controls how far a single module update can move from the
merged weights. This objective is a tractable module-local surrogate for
reducing linear module feature drift, rather than an exact objective for
end-to-end task risk or all-layer feature drift.

\subsection{Feature Interpolation for Calibration Targets}

Under the forward-order schedule, each layer's calibration should focus on the
local mismatch at its current linear module rather than compensate for feature
drift caused by upstream layers. Directly using expert input features for
calibration can violate this goal. The gap between \(X_i^{\cali}\) and
\(X_i^{\expert}\) already contains upstream feature drift, so fitting the direct
target \(W_iX_i^{\expert}\) may force the current linear module to fit drift left
by earlier layers. We therefore introduce an interpolated target feature:
\begin{equation}
\label{eq:featcal-target-feature}
X_i^{\tgt}
=
\alpha X_i^{\expert}
+
(1-\alpha) X_i^{\cali},
\qquad
\alpha \in [0,1].
\end{equation}
In \cref{eq:featcal-basic}, we replace the target term
\(W_iX_i^{\expert}\) with \(W_iX_i^{\tgt}\).
This target keeps the expert signal while making the local regression less
aggressive under upstream drift.

\subsection{Anchor Regularization}

During calibration, the objective should use targets formed from expert features
while keeping the update tied to the merged weights. The base model provides
another useful reference. It is pretrained on large data before task
specialization and can contain knowledge that a small calibration set does not
cover. To use this reference, we include the base weight in the regularization
term. The coefficients below let us control how calibration uses the merged and
base references. For a linear module, we define the anchor weight as
\begin{equation}
\label{eq:featcal-anchor}
W^{\anchor}
=
\rho W^{\merge}
+
(1-\rho) W^{\base},
\qquad
\rho\in\mathbb{R}.
\end{equation}
The coefficient \(\rho\) controls the anchor used by the quadratic penalty.

Combining the target interpolation in \cref{eq:featcal-target-feature} and the
anchor in \cref{eq:featcal-anchor} gives the practical objective:
\begin{equation}
\label{eq:featcal-practical}
W^\star
=
\arg\min_{W \in \mathbb{R}^{m \times d}}
\sum_{i=1}^N
\| W X_i^{\cali} - W_i X_i^{\tgt} \|_F^2
\;+\;
\lambda
\| W - W^{\anchor} \|_F^2.
\end{equation}
The first term uses the interpolated target from \cref{eq:featcal-target-feature}
to fit the expert feature signal without using the raw expert input feature as a
hard target. The second term uses the anchor from \cref{eq:featcal-anchor} to
keep the update close to the chosen reference, with \(\lambda>0\) controlling
the regularization strength.

\subsection{Task-Wise Scale Normalization and Closed-Form Solution}

This subsection describes the update applied after feature collection. At each
forward-order layer, \method first caches current calibrated-model features and
expert features for the modules calibrated in that layer. The cached features
are then used to compute each linear module update separately, and the layer
parameters are loaded after the layer's updates are formed. Thus the statistics
below are module-local, even though feature collection is organized by layer.

For a fixed linear module in the current layer, let \(n\) denote the calibration
sample count for each task at this module. The \(1/n\) factors form per-task
empirical moments and prevent each task contribution from scaling directly with
the sample count. We summarize task \(i\) by the empirical feature statistics
\begin{equation}
\label{eq:featcal-stats}
G_i
=
\frac{1}{n}
X_i^{\cali} {X_i^{\cali}}^\top
\in \mathbb{R}^{d \times d},
\qquad
C_i
=
\frac{1}{n}
X_i^{\tgt} {X_i^{\cali}}^\top
\in \mathbb{R}^{d \times d}.
\end{equation}
Here, \(G_i\) is the input second moment and \(C_i\) is the target-input cross moment.

The objective in \cref{eq:featcal-practical} is a matrix-valued ridge regression problem~\citep{hoerl1970ridge,tikhonov1977solutions}. To reduce scale sensitivity, we use stabilized inverse task weights with \(\epsilon>0\),
\begin{equation}
\label{eq:featcal-task-weight}
\nu_i = \max\{\|G_i\|_F,\epsilon\},
\qquad
\omega_i = \nu_i^{-1},
\end{equation}
With these stabilized task weights fixed, the module-wise objective becomes
\begin{equation}
\label{eq:featcal-final-obj}
W^\star
=
\arg\min_{W \in \mathbb{R}^{m \times d}}
\sum_{i=1}^N
\frac{\omega_i}{n}
\| W X_i^{\cali} - W_i X_i^{\tgt} \|_F^2
\;+\;
\lambda
\| W - W^{\anchor} \|_F^2.
\end{equation}
The corresponding stationary condition for this quadratic objective is
\begin{equation}
\label{eq:featcal-final-stationary}
W \Bigl( \sum_{i=1}^N \omega_i G_i + \lambda I_d \Bigr)
=
\sum_{i=1}^N \omega_i W_i C_i + \lambda W^{\anchor},
\end{equation}
Solving this linear system gives the closed-form update for this module
\begin{equation}
\label{eq:featcal-final-closed-form}
W^\star
=
\Bigl(
\sum_{i=1}^N \omega_i W_i C_i
+
\lambda W^{\anchor}
\Bigr)
\Bigl(
\sum_{i=1}^N \omega_i G_i
+
\lambda I_d
\Bigr)^{-1}.
\end{equation}
Because \(G_i\succeq 0\), \(\omega_i>0\), and \(\lambda>0\), the inverse is well defined. In implementation, we also add \(\epsilon I_d\) to the solve matrix as a numerical stabilizer.

The complete forward-order calibration procedure is given in \cref{app:calibration-algorithm}.

\section{Experiments}
\label{sec:experiments}

\subsection{Setup}
\label{sec:exp-setup}

\paragraph{Benchmarks.}
We use two public FusionBench settings~\citep{tang2025fusionbench} and one MergeBench LLM setting~\citep{he2025mergebench}:
\par\noindent\ding{182}\enspace \textbf{CLIP image classification.} We use the FusionBench CLIP model merging benchmark with CLIP-ViT-B/32 and CLIP-ViT-L/14~\citep{radford2021learning}. The primary setting has 8 image tasks: SUN397, Stanford Cars, RESISC45, EuroSAT, SVHN, GTSRB, MNIST, and DTD~\citep{xiao2010sun,krause2013cars,cheng2017remote,helber2019eurosat,netzer2011reading,stallkamp2011german,lecun1998gradient,cimpoi2014describing}. The 14-task suite adds Flowers102, PCAM, FER2013, Oxford-IIIT Pet, STL10, and CIFAR100~\citep{nilsback2008automated,veeling2018rotation,goodfellow2015challenges,parkhi2012cats,coates2011analysis,krizhevsky2009learning}. The 20-task suite further adds CIFAR10, Food101, Fashion-MNIST, EMNIST Letters, KMNIST, and Rendered SST2~\citep{krizhevsky2009learning,bossard2014food,xiao2017fashion,cohen2017emnist,clanuwat2018deep,socher2013recursive,openai2021renderedsst2}. We follow prior merging protocols~\citep{yang2024adamerging,cheng2025wudi}, report top-1 accuracy and task averages, and defer full per-task extended results to \cref{app:clip-detailed-results}.
\par\noindent\ding{183}\enspace \textbf{FLAN-T5 text generation.} We evaluate FusionBench FLAN-T5-base and FLAN-T5-large merging on 8 prompted GLUE tasks: CoLA, MNLI, MRPC, QNLI, QQP, RTE, SST-2, and STS-B~\citep{raffel2020t5,wei2022finetuned,chung2024scaling,wang2018glue,warstadt2019neural,williams2018broad,dolan2005automatically,rajpurkar2016squad,dagan2006pascal,socher2013recursive,cer2017semeval}. The base experts are full fine-tuned models, while the large experts use LoRA fine-tuning~\citep{hu2022lora}. We merge task experts and evaluate generated text outputs. We report exact match accuracy except for STS-B, where we report Spearman's \(\rho\), and average the 8 task scores.
\par\noindent\ding{184}\enspace \textbf{MergeBench LLM merging.} We evaluate Llama-3.2-3B-Instruct and Llama-3.1-8B-Instruct in the MergeBench domain-expert setting. The task suite covers mathematics, coding, instruction following, and general knowledge through MATH-500, GSM8K, HumanEval+, MBPP+, IFEval, and ARC-Challenge. HumanEval+ and MBPP+ report pass@1, and each table average is the mean over all 6 reported MergeBench tasks in this LLM setting.

\paragraph{Compared methods.}
We compare against pre-trained, single-task, and multi-task references when available. For CLIP, upstream mergers include Simple Averaging~\citep{wortsman2022modelsoups}, Task Arithmetic~\citep{ilharco2023taskarithmetic}, AdaMerging~\citep{yang2024adamerging}, and WUDI-Merging~\citep{cheng2025wudi}; for FLAN-T5 and MergeBench, we use Task Arithmetic. We apply \method on top of different upstream mergers and compare with Surgery~\citep{yang2024surgery} and ProbSurgery~\citep{wei2025probsurgery} where available.
Unless otherwise stated, upstream and baseline hyperparameters follow the FusionBench recipes.

\paragraph{Calibration setup.}
By default, \method uses 256 calibration samples per task, calibrates layers in forward order, and applies the linear-weight update in \cref{eq:featcal-final-closed-form}. When enabled, it also applies the bias and LayerNorm affine updates in Appendix~\ref{app:affine-layernorm-calibration}; the main CLIP runs enable both, while the FLAN-T5 runs calibrate linear-module bias parameters but not LayerNorm affine parameters. For the main CLIP accuracy tables, we fix \(\lambda=0.05\), \(\rho=2.0\), and \(\alpha=0.3\). For FLAN-T5, we use \(\lambda=10^{-5}\), \(\rho=2.0\), and \(\alpha=0.6\). For MergeBench, we use \((\lambda,\rho,\alpha)=(10^{-5},0.5,0.15)\) for Llama-3.2-3B-Instruct and \((3.0,0.5,0.15)\) for Llama-3.1-8B-Instruct. Test sets are used only for final reporting. \Cref{sec:analysis} includes a compact sensitivity diagnostic for the 8-task TA setting.

\subsection{Results on CLIP Models}
\label{sec:clip-results}

\begin{table}[!t]
\centering
\caption{
Multi-task performance of CLIP-ViT-B/32 models on 8 image-classification tasks.
All numbers are top-1 accuracy (\%). Gray arrows show changes over upstream mergers.
}
\label{tab:clip-vit-b32-results}
\setlength{\tabcolsep}{3.2pt}
\resizebox{\textwidth}{!}{
\begin{tabular}{l|cccccccc|c}
\toprule
Method & SUN397 & Cars & RESISC45 & EuroSAT & SVHN & GTSRB & MNIST & DTD & Avg. \\
\midrule
Pre-trained & 63.2 & 59.8 & 60.7 & 46.0 & 31.6 & 32.5 & 48.2 & 43.9 & 48.2 \\
Fine-tuned (STL) & 75.0 & 78.3 & 95.2 & 99.0 & 97.3 & 98.9 & 99.6 & 79.7 & 90.3 \\
Traditional MTL & 72.3 & 76.6 & 92.2 & 97.9 & 95.5 & 97.7 & 99.3 & 77.7 & 88.6 \\
\midrule
Simple Averaging & 65.4\zerogain & 62.4\zerogain & 70.6\zerogain & 75.7\zerogain & 64.5\zerogain & 55.0\zerogain & 86.3\zerogain & 50.6\zerogain & 66.3\zerogain \\
\quad w/ Surgery & 67.4\upgain{2.0} & 63.5\upgain{1.1} & 80.5\upgain{9.9} & 94.7\upgain{19.} & 70.8\upgain{6.3} & 79.7\upgain{24.} & 96.8\upgain{10.} & 64.7\upgain{14.} & 77.3\upgain{11.} \\
\quad w/ ProbSurgery & 69.2\upgain{3.8} & 65.8\upgain{3.4} & 85.8\upgain{15.} & 93.4\upgain{17.} & 70.4\upgain{5.9} & 87.0\upgain{32.} & 96.5\upgain{10.} & 69.1\upgain{18.} & 79.7\upgain{13.} \\
\rowcolor{FeatCalRow}\quad w/ \method & 69.7\upgain{4.3} & 70.4\upgain{8.0} & 85.0\upgain{14.} & 95.4\upgain{19.} & 92.6\upgain{28.} & 87.1\upgain{32.} & 97.9\upgain{11.} & 67.4\upgain{16.} & 83.2\upgain{16.} \\
\midrule
Task Arithmetic & 57.0\zerogain & 55.7\zerogain & 64.7\zerogain & 73.3\zerogain & 77.9\zerogain & 68.5\zerogain & 96.1\zerogain & 47.1\zerogain & 67.5\zerogain \\
\quad w/ Surgery & 59.9\upgain{2.9} & 59.9\upgain{4.2} & 76.1\upgain{11.} & 92.4\upgain{19.} & 83.6\upgain{5.7} & 84.9\upgain{16.} & 98.0\upgain{1.9} & 61.1\upgain{14.} & 77.0\upgain{9.5} \\
\quad w/ ProbSurgery & 60.6\upgain{3.6} & 61.1\upgain{5.4} & 81.0\upgain{16.} & 93.8\upgain{20.} & 85.3\upgain{7.4} & 87.1\upgain{18.} & 97.9\upgain{1.8} & 63.8\upgain{16.} & 78.8\upgain{11.} \\
\rowcolor{FeatCalRow}\quad w/ \method & 70.1\upgain{13.} & 72.5\upgain{16.} & 88.1\upgain{23.} & 96.3\upgain{23.} & 95.0\upgain{17.} & 93.0\upgain{24.} & 98.8\upgain{2.7} & 69.8\upgain{22.} & 85.5\upgain{18.} \\
\midrule
AdaMerging & 67.9\zerogain & 71.2\zerogain & 84.0\zerogain & 92.3\zerogain & 87.6\zerogain & 93.1\zerogain & 98.2\zerogain & 66.9\zerogain & 82.7\zerogain \\
\quad w/ Surgery & 69.8\upgain{1.9} & 72.1\upgain{0.9} & 88.7\upgain{4.7} & 95.3\upgain{3.0} & 90.5\upgain{2.9} & 95.7\upgain{2.6} & 98.7\upgain{0.5} & 73.3\upgain{6.4} & 85.5\upgain{2.8} \\
\quad w/ ProbSurgery & 70.6\upgain{2.7} & 72.9\upgain{1.7} & 90.4\upgain{6.4} & 95.8\upgain{3.5} & 90.2\upgain{2.6} & 95.0\upgain{1.9} & 98.8\upgain{0.6} & 73.8\upgain{6.9} & 85.9\upgain{3.2} \\
\rowcolor{FeatCalRow}\quad w/ \method & \textbf{73.0}\upgain{5.1} & \textbf{77.4}\upgain{6.2} & 91.3\upgain{7.3} & 96.8\upgain{4.5} & 94.2\upgain{6.6} & 96.7\upgain{3.6} & 99.0\upgain{0.8} & 76.4\upgain{9.5} & 88.1\upgain{5.4} \\
\midrule
WUDI-Merging & 68.0\zerogain & 72.5\zerogain & 85.0\zerogain & 94.6\zerogain & 94.8\zerogain & 95.0\zerogain & 99.3\zerogain & 66.7\zerogain & 84.5\zerogain \\
\quad w/ Surgery & 69.0\upgain{1.0} & 71.7\downgain{0.8} & 89.1\upgain{4.1} & 97.2\upgain{2.6} & 95.6\upgain{0.8} & 96.7\upgain{1.7} & 99.3\upgain{0.0} & 72.7\upgain{6.0} & 86.4\upgain{1.9} \\
\quad w/ ProbSurgery & 69.5\upgain{1.5} & 72.7\upgain{0.2} & 90.6\upgain{5.6} & 97.3\upgain{2.7} & 95.6\upgain{0.8} & 97.1\upgain{2.1} & 99.3\upgain{0.0} & 73.3\upgain{6.6} & 86.9\upgain{2.4} \\
\rowcolor{FeatCalRow}\quad w/ \method & 72.2\upgain{4.2} & 76.0\upgain{3.5} & \textbf{93.0}\upgain{8.0} & \textbf{98.1}\upgain{3.5} & \textbf{96.7}\upgain{1.9} & \textbf{97.9}\upgain{2.9} & \textbf{99.4}\upgain{0.1} & \textbf{77.1}\upgain{10.} & \textbf{88.8}\upgain{4.3} \\
\bottomrule
\end{tabular}
}
\end{table}

\begin{wraptable}[12]{r}{0.53\textwidth}
\centering
\vspace{-2.0em}
\caption{\small CLIP Extended Average Accuracy.}
\label{tab:clip-extended-avg}
\vspace{0.15em}
\scriptsize
\setlength{\tabcolsep}{1.7pt}
\renewcommand{\arraystretch}{0.78}
\resizebox{\linewidth}{!}{
\begin{tabular}{l|ccccc}
\toprule
\multirow{2}{*}{Method} & 8-Task & 14-Task & 14-Task & 20-Task & 20-Task \\
& L/14 & B/32 & L/14 & B/32 & L/14 \\
\midrule
Pre-trained & 64.6 & 58.8 & 69.1 & 55.6 & 65.6 \\
Fine-tuned (STL) & 94.3 & 90.0 & 92.8 & 90.3 & 93.1 \\
\midrule
Task Arithmetic & 80.5\zerogain & 66.2\zerogain & 77.3\zerogain & 60.6\zerogain & 70.3\zerogain \\
\quad w/ Surgery & 86.0\upgain{5.5} & 76.8\upgain{10.} & 83.8\upgain{6.5} & 75.2\upgain{14.} & 82.4\upgain{12.} \\
\quad w/ ProbSurgery & 87.6\upgain{7.1} & 78.3\upgain{12.} & 86.1\upgain{8.8} & 77.7\upgain{17.} & 83.3\upgain{13.} \\
\rowcolor{FeatCalRow}\quad w/ \method & 91.6\upgain{11.} & 81.5\upgain{15.} & 90.7\upgain{13.} & 79.4\upgain{18.} & 84.7\upgain{14.} \\
\midrule
WUDI-Merging & 92.2\zerogain & 78.7\zerogain & 88.8\zerogain & 67.1\zerogain & 75.8\zerogain \\
\quad w/ Surgery & 92.8\upgain{0.6} & 82.5\upgain{3.8} & 90.3\upgain{1.5} & 76.5\upgain{9.4} & 84.9\upgain{9.1} \\
\quad w/ ProbSurgery & 93.0\upgain{0.8} & 82.7\upgain{4.0} & 90.7\upgain{1.9} & 77.6\upgain{10.} & 86.5\upgain{10.} \\
\rowcolor{FeatCalRow}\quad w/ \method & \textbf{93.5}\upgain{1.3} & \textbf{86.2}\upgain{7.5} & \textbf{91.5}\upgain{2.7} & \textbf{83.5}\upgain{16.} & \textbf{89.8}\upgain{14.} \\
\bottomrule
\end{tabular}
}
\vspace{-0.7em}
\end{wraptable}

On B/32 8-task CLIP, \method raises the 4 upstream averages from 66.3/67.5/82.7/84.5 to 83.2/85.5/88.1/88.8, beating Surgery and ProbSurgery in each block. The gains are largest for the weaker upstream mergers, but \method also improves AdaMerging and WUDI-Merging, where the merged models are already close to the multi-task reference. This pattern aligns with the feature drift motivation: the same calibration step can recover large lost accuracy and still refine strong merged models without changing the merger itself. The best average, 88.8, is close to the 90.3 task expert average and above the 88.6 multi-task reference. For the TA and WUDI rows in \Cref{tab:clip-extended-avg}, \method remains above Surgery and ProbSurgery across all extended averages. On 20-task TA, \method reaches 79.4 on B/32 and 84.7 on L/14, giving \(+18.8\) and \(+14.4\) points over TA. Full per-task results are in \cref{app:clip-detailed-results}.

\subsection{Results on FLAN-T5 Models}
\label{sec:flan-t5-results}

\begin{table}[!htbp]
\centering
\caption{
FLAN-T5 GLUE generation results.
Scores are percentages: exact-match accuracy except STS-B, which reports Spearman's \(\rho\); arrows show changes over Task Arithmetic, and bold marks best non-reference entries.
}
\label{tab:flan-t5-results}
\setlength{\tabcolsep}{2.9pt}
\resizebox{\textwidth}{!}{
\begin{tabular}{l|l|cccccccc|c}
\toprule
\multirow{2}{*}{Model} & \multirow{2}{*}{Method} & \multicolumn{8}{c|}{GLUE Tasks} & \multirow{2}{*}{Avg.} \\
\cmidrule(lr){3-10}
& & CoLA & MNLI & MRPC & QNLI & QQP & RTE & SST-2 & STS-B & \\
\midrule
\multirow{6}{*}{Base}
& Pre-trained & 69.1 & 56.5 & 76.2 & 88.4 & 82.1 & 80.1 & 91.2 & 62.2 & 75.7 \\
& Fine-tuned (STL) & 75.0 & 83.4 & 87.5 & 91.5 & 85.4 & 85.9 & 93.6 & 88.7 & 86.4 \\
\cmidrule(lr){2-11}
& Task Arithmetic & 70.5\zerogain & 57.8\zerogain & 78.4\zerogain & 90.2\zerogain & 83.6\zerogain & 80.5\zerogain & 92.3\zerogain & 77.8\zerogain & 78.9\zerogain \\
& \quad w/ Surgery & 70.8\upgain{0.3} & 82.4\upgain{24.} & 82.4\upgain{4.0} & 89.8\downgain{0.4} & 84.2\upgain{0.6} & 83.0\upgain{2.5} & 92.1\downgain{0.2} & 85.2\upgain{7.4} & 83.7\upgain{4.8} \\
& \quad w/ ProbSurgery & \textbf{82.4}\upgain{11.} & 69.1\upgain{11.} & 78.3\downgain{0.1} & 80.6\downgain{9.6} & \textbf{89.8}\upgain{6.2} & 83.4\upgain{2.9} & 81.2\downgain{11.} & \textbf{92.5}\upgain{14.} & 82.2\upgain{3.3} \\
\rowcolor{FeatCalRow}
& \quad w/ \method & 72.2\upgain{1.7} & \textbf{82.6}\upgain{24.} & \textbf{85.1}\upgain{6.7} & \textbf{91.1}\upgain{0.9} & 84.7\upgain{1.1} & \textbf{85.2}\upgain{4.7} & \textbf{93.0}\upgain{0.7} & 87.7\upgain{9.9} & \textbf{85.2}\upgain{6.3} \\
\midrule
\multirow{6}{*}{Large}
& Pre-trained & 73.7 & 56.6 & 82.4 & 91.1 & 85.5 & 85.6 & 94.3 & 87.5 & 82.1 \\
& Fine-tuned (STL) & 80.2 & 88.5 & 89.2 & 94.4 & 87.2 & 91.7 & 95.2 & 90.9 & 89.6 \\
\cmidrule(lr){2-11}
& Task Arithmetic & 76.8\zerogain & 85.4\zerogain & 85.3\zerogain & \textbf{94.0}\zerogain & 85.8\zerogain & 88.1\zerogain & 95.2\zerogain & 87.7\zerogain & 87.3\zerogain \\
& \quad w/ Surgery & 76.0\downgain{0.8} & \textbf{87.9}\upgain{2.5} & 86.0\upgain{0.7} & 93.9\downgain{0.1} & 86.2\upgain{0.4} & \textbf{89.9}\upgain{1.8} & 95.2\upgain{0.0} & 89.0\upgain{1.3} & 88.0\upgain{0.7} \\
& \quad w/ ProbSurgery & 77.2\upgain{0.4} & 87.6\upgain{2.2} & 85.0\downgain{0.3} & 93.8\downgain{0.2} & 86.0\upgain{0.2} & 88.4\upgain{0.3} & 95.2\upgain{0.0} & 87.8\upgain{0.1} & 87.6\upgain{0.3} \\
\rowcolor{FeatCalRow}
& \quad w/ \method & \textbf{79.2}\upgain{2.4} & 87.8\upgain{2.4} & \textbf{89.0}\upgain{3.7} & 93.9\downgain{0.1} & \textbf{86.9}\upgain{1.1} & 88.4\upgain{0.3} & \textbf{95.3}\upgain{0.1} & \textbf{90.8}\upgain{3.1} & \textbf{88.9}\upgain{1.6} \\
\bottomrule
\end{tabular}
}
\end{table}

\Cref{tab:flan-t5-results} shows that the GLUE gains hold for both FLAN-T5-base and FLAN-T5-large. For base, \method improves Task Arithmetic by \(+6.3\) average points, exceeds Surgery and ProbSurgery, and improves all 8 tasks, with the largest gains on MNLI and STS-B. For large, where LoRA-based Task Arithmetic is already strong, \method gives smaller gains but still reaches the best post-TA average. This contrast suggests that feature calibration helps most when the merged generator has clear headroom, while still helping in the stronger setting.

\subsection{Results on MergeBench LLMs}
\label{sec:mergebench-results}

\begin{table}[!htbp]
\centering
\caption{
MergeBench results on Llama-family models.
All Task Arithmetic rows use scale \(0.3\); arrows show changes over Task Arithmetic within each model block.
Bold marks the best non-reference entry for each metric within the corresponding model block.
}
\label{tab:mergebench-llama-results}
\setlength{\tabcolsep}{3.8pt}
\resizebox{\textwidth}{!}{
\begin{tabular}{ll|cccccc|c}
\toprule
Model & Method & MATH-500 & GSM8K & HumanEval+ & MBPP+ & IFEval & ARC-C & Avg. \\
\midrule
\multirow{6}{*}{Llama-3.2-3B-Instruct}
& Base & 47.8 & 72.6 & 49.4 & 56.6 & 67.3 & 73.1 & 61.1 \\
& Fine-tuned (STL) & 49.0 & 80.6 & 54.9 & 57.1 & 69.7 & 73.1 & 64.1 \\
\cmidrule(lr){2-9}
& Task Arithmetic & 44.0\zerogain & 74.8\zerogain & 52.4\zerogain & 53.4\zerogain & 63.2\zerogain & 72.6\zerogain & 60.1\zerogain \\
& \quad w/ Surgery & 45.8\upgain{1.8} & 74.5\downgain{0.3} & 53.0\upgain{0.6} & 54.5\upgain{1.1} & 64.3\upgain{1.1} & 72.9\upgain{0.3} & 60.8\upgain{0.7} \\
& \quad w/ ProbSurgery & 46.0\upgain{2.0} & 74.8\upgain{0.0} & \textbf{54.3}\upgain{1.9} & 54.5\upgain{1.1} & 64.1\upgain{0.9} & \textbf{73.0}\upgain{0.4} & 61.1\upgain{1.0} \\
\rowcolor{FeatCalRow}
& \quad w/ \method & \textbf{47.0}\upgain{3.0} & \textbf{78.0}\upgain{3.2} & 53.0\upgain{0.6} & \textbf{55.3}\upgain{1.9} & \textbf{66.5}\upgain{3.3} & 72.5\downgain{0.1} & \textbf{62.1}\upgain{2.0} \\
\midrule
\multirow{6}{*}{Llama-3.1-8B-Instruct}
& Base & 48.6 & 83.5 & 62.2 & 63.0 & 72.8 & 80.9 & 68.5 \\
& Fine-tuned (STL) & 52.6 & 85.3 & 67.1 & 63.8 & 72.8 & 80.9 & 70.4 \\
\cmidrule(lr){2-9}
& Task Arithmetic & 49.0\zerogain & \textbf{85.6}\zerogain & 62.8\zerogain & 61.6\zerogain & 45.3\zerogain & 76.9\zerogain & 63.5\zerogain \\
& \quad w/ Surgery & 47.2\downgain{1.8} & 85.2\downgain{0.4} & \textbf{64.6}\upgain{1.8} & 62.2\upgain{0.6} & 47.1\upgain{1.8} & 77.6\upgain{0.7} & 64.0\upgain{0.5} \\
& \quad w/ ProbSurgery & \textbf{49.6}\upgain{0.6} & 85.3\downgain{0.3} & 62.8\upgain{0.0} & 61.9\upgain{0.3} & 48.8\upgain{3.5} & 77.8\upgain{0.9} & 64.4\upgain{0.9} \\
\rowcolor{FeatCalRow}
& \quad w/ \method & 47.6\downgain{1.4} & 84.5\downgain{1.1} & 59.8\downgain{3.0} & \textbf{62.7}\upgain{1.1} & \textbf{60.8}\upgain{15.} & \textbf{79.5}\upgain{2.6} & \textbf{65.8}\upgain{2.3} \\
\bottomrule
\end{tabular}
}
\end{table}

\Cref{tab:mergebench-llama-results} shows that the gains extend to LLM domain expert merging. \method improves the 6-task average over Task Arithmetic by \(+2.0\) points on the 3B model and \(+2.3\) points on the 8B model, outperforming Surgery and ProbSurgery in both model blocks. The largest gain is on IFEval for the 8B setting, where \method improves Task Arithmetic by \(+15.\) points.

\subsection{Analysis}
\label{sec:analysis}

We collect 4 diagnostics that probe the mechanism, practical cost, robustness, and stability of post-merging feature calibration.

\begin{figure}[t]
\centering
\vspace{-0.5em}
\includegraphics[width=\textwidth]{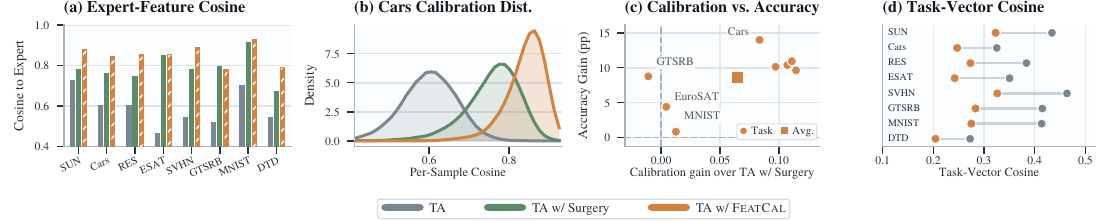}
\vspace{-1.15em}
\caption{
Feature-calibration diagnostics for TA on CLIP-ViT-B/32. 
(a) Task-wise final-feature cosine to experts; (b) per-sample cosine on Cars; (c) expert-cosine gain versus accuracy gain over TA w/ Surgery; (d) backbone task-vector cosine to same-task experts.
}
\label{fig:feature-calibration-diagnostics}
\end{figure}

\paragraph{Feature-calibration diagnostics.}
\Cref{fig:feature-calibration-diagnostics} tests whether TA w/ \method{} gains coincide with final features closer to task experts. Panels (a) and (b) show larger expert cosine than Surgery in this setting: the macro average rises from 0.785 to 0.850 over TA w/ Surgery, and Stanford Cars has a mean per-sample gain of 0.084. This matches the intended mechanism because \method{} calibrates the feature distribution reached by the deployed merged model. Panel (c) shows that expert-feature cosine is a diagnostic rather than a complete explanation: larger cosine gains often track accuracy gains, with an average \(+8.64\)-point gain over TA w/ Surgery, but EuroSAT and MNIST improve with small cosine changes, while GTSRB improves despite a slight decrease. Panel (d) highlights a feature-space and parameter-space mismatch: although \method{} moves final features closer to experts, the calibrated backbone has lower same-task expert task-vector cosine than the TA baseline on every task. Thus, \method{} calibrates the deployed feature distribution without bringing the backbone parameters into closer task-vector alignment, and the ProbSurgery comparison in \cref{app:probsurgery-diagnostic} shows the same pattern with a stronger adapter baseline.

\begin{wrapfigure}[21]{r}{0.46\textwidth}
\centering
\vspace{-1.8em}
\captionof{table}{\small
Post-TA runtime at \(n=256\), excluding final evaluation. Speedup is relative to Surgery.
}
\label{tab:ta8-runtime-efficiency-n256}
\vspace{0.15em}
{\setlength{\tabcolsep}{1.7pt}
\renewcommand{\arraystretch}{1.0}
\newcommand{\runtimecell}[2]{\makebox[4.35em][c]{#1$\times$ #2}}
\small
\begin{tabular}{@{}lccc@{}}
\toprule
\multirow{2}{*}{Method}
& Speedup
& GPU Energy
& CPU RSS \\
& (Time) & (Wh) & (GiB) \\
\midrule
Surgery & \runtimecell{1.0}{(217s)} & 18.1 & 69.3 \\
ProbSurgery & \runtimecell{1.0}{(224s)} & 18.3 & 87.0 \\
\rowcolor{FeatCalRow} \method & \runtimecell{\textbf{4.1}}{\textbf{(053s)}} & \textbf{\phantom{0}1.9} & \textbf{22.8} \\
\bottomrule
\end{tabular}
}
\vspace{-0.65em}
\includegraphics[width=0.98\linewidth]{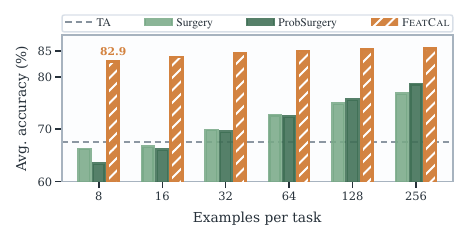}
\vspace{-0.45em}
\captionof{figure}{\small
Sample Efficiency.
}
\label{fig:ta8-data-efficiency}
\vspace{-0.2em}
\end{wrapfigure}

\paragraph{Sample efficiency and calibration cost.}
\Cref{fig:ta8-data-efficiency} reports post-TA sample efficiency on CLIP-ViT-B/32 TA with matched per-task budgets \(n\). All methods start from the same Task Arithmetic model and differ only in the post-merging calibration stream.
With \(8\) examples per task, \method already reaches \(82.9\%\) average accuracy and then saturates near \(85.5\%\), while Surgery and ProbSurgery rise more slowly.
\Cref{tab:ta8-runtime-efficiency-n256} fixes \(n=256\): excluding final evaluation, the closed-form update takes 53s, \(4.1\times\) faster than Surgery and \(4.2\times\) faster than ProbSurgery, with lower GPU energy and much less peak CPU RSS.
Together, \method reaches most of its TA gain with few examples and keeps calibration cost low at \(n=256\).
Full per-budget accuracy, resource averages, hardware, and logging details are in \cref{app:efficiency-protocol}.
\WFclear

\paragraph{Robustness to corrupted calibration data.}
\begin{wrapfigure}[13]{r}{0.52\textwidth}
\centering
\vspace{-1.55em}
\captionof{table}{\small
8-task TA accuracy under corrupted calibration. Avg. includes clean, Gaussian, blur, and fog.
}
\label{tab:ta8-corrupted-calibration}
\vspace{0.2em}
{\setlength{\tabcolsep}{1.8pt}
\renewcommand{\arraystretch}{0.92}
\footnotesize
\begin{tabularx}{\linewidth}{@{}>{\raggedright\arraybackslash}p{0.31\linewidth}*{5}{>{\centering\arraybackslash}X}@{}}
\toprule
Method & Clean & Gauss. & Blur & Fog & Avg. \\
\midrule
Task Arithmetic & 67.5 & -- & -- & -- & -- \\
\quad w/ Surgery & 76.8 & 72.4 & 67.5 & 69.4 & 71.5 \\
\quad w/ ProbSurgery & 79.1 & 73.2 & 65.8 & 68.2 & 71.6 \\
\rowcolor{FeatCalRow}\quad w/ \method & \textbf{85.5} & \textbf{76.6} & \textbf{74.2} & \textbf{75.6} & \textbf{78.0} \\
\bottomrule
\end{tabularx}
}
\vspace{-0.3em}
\includegraphics[width=1.0\linewidth]{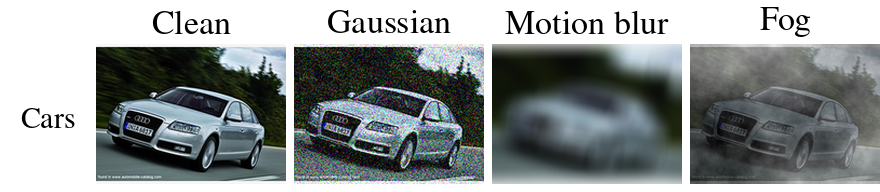}
\vspace{-1.55em}
\captionof{figure}{\small
Calibration examples under corruptions.
}
\label{fig:calibration-corruption-examples}
\vspace{-0.9em}
\end{wrapfigure}

We corrupt only the images used for post-merging calibration and evaluate on clean test sets from 8-task TA, isolating calibration data quality rather than test-time corruption robustness. Protocol details are in \cref{app:corrupted-calibration-protocol}.

\method remains strongest in every reported setting, with a \(78.0\%\) average over clean, Gaussian noise, motion blur, and fog, compared with \(71.5\%\) for Surgery and \(71.6\%\) for ProbSurgery.
Corrupted runs remain below the clean reference, but the stable ordering in \cref{fig:calibration-corruption-examples} suggests that \method extracts a reliable expert-feature calibration signal from imperfect data.
\WFclear

\Needspace{0.62\textheight}
\paragraph{Ablation study.}
\begin{wrapfigure}[16]{r}{0.369\textwidth}
\centering
\vspace{0em}
\includegraphics[width=\linewidth]{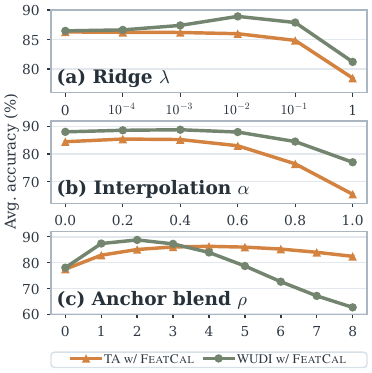}
\vspace{-1.5em}
\caption{\small
CLIP-ViT-B/32 TA coefficient sweeps.
}
\label{fig:featcal-clip-ablation}
\vspace{0em}
\end{wrapfigure}

\Cref{fig:featcal-clip-ablation} reports single-factor sensitivity in the 8-task TA setting.
\method works best in a conservative calibration regime. Overly aggressive targets or regularization can degrade performance.
The ridge sweep is flat over small positive values, so the closed-form solve is not tied to a narrow regularization choice, while very large ridge strength suppresses the update.
The interpolation sweep has a sharper target-side tradeoff: small and medium \(\alpha\) values help, whereas the hard-expert endpoint can force local updates to absorb upstream drift that should be calibrated gradually across layers.
The merged-base blend curve is most asymmetric: moderate extrapolation improves both upstream mergers, but large \(\rho\) values push calibrated weights too far from the merged-base reference and especially hurt the WUDI variant.
Overall, the ablation is a stability diagnostic with a broad conservative region rather than a single sharp optimum.
\WFclear

\section{Conclusion}

Feature drift frames post-merging calibration by showing how local mismatches can propagate through a merged model and affect outputs. \method uses this signal by capturing features layer by layer and applying closed-form updates to linear modules, improving CLIP and FLAN-T5 mergers over Surgery and ProbSurgery in the main settings without adding modules at inference.


\section*{Acknowledgments}

This work was supported by the Hong Kong RGC (TRS: T41-517/25-N; GRF: 15228325), ITC RAISe+ (RAI/24/1/086A), and the Research Institute for Generative Artificial Intelligence at PolyU.

\bibliographystyle{unsrtnat}
\bibliography{neurips_2026}

\begin{thebibliography}{64}
\providecommand{\natexlab}[1]{#1}
\providecommand{\url}[1]{\texttt{#1}}
\expandafter\ifx\csname urlstyle\endcsname\relax
  \providecommand{\doi}[1]{doi: #1}\else
  \providecommand{\doi}{doi: \begingroup \urlstyle{rm}\Url}\fi

\bibitem[Wortsman et~al.(2022)Wortsman, Ilharco, Gadre, Roelofs, Gontijo-Lopes,
  Morcos, Namkoong, Farhadi, Carmon, Kornblith, and
  Schmidt]{wortsman2022modelsoups}
Mitchell Wortsman, Gabriel Ilharco, Samir~Ya Gadre, Rebecca Roelofs, Raphael
  Gontijo-Lopes, Ari~S. Morcos, Hongseok Namkoong, Ali Farhadi, Yair Carmon,
  Simon Kornblith, and Ludwig Schmidt.
\newblock Model soups: averaging weights of multiple fine-tuned models improves
  accuracy without increasing inference time.
\newblock In \emph{ICML}, 2022.

\bibitem[Matena and Raffel(2022)]{matena2022fisher}
Michael~S. Matena and Colin~A. Raffel.
\newblock Merging models with fisher-weighted averaging.
\newblock In \emph{NeurIPS}, 2022.

\bibitem[Ilharco et~al.(2023)Ilharco, Ribeiro, Wortsman, Schmidt, Hajishirzi,
  and Farhadi]{ilharco2023taskarithmetic}
Gabriel Ilharco, Marco~Tulio Ribeiro, Mitchell Wortsman, Ludwig Schmidt,
  Hannaneh Hajishirzi, and Ali Farhadi.
\newblock Editing models with task arithmetic.
\newblock In \emph{ICLR}, 2023.

\bibitem[Yadav et~al.(2023)Yadav, Tam, Choshen, Raffel, and
  Bansal]{yadav2023ties}
Prateek Yadav, Derek Tam, Leshem Choshen, Colin Raffel, and Mohit Bansal.
\newblock {TIES-Merging}: Resolving interference when merging models.
\newblock In \emph{NeurIPS}, 2023.

\bibitem[Yang et~al.(2024{\natexlab{a}})Yang, Wang, Shen, Liu, Guo, Wang, and
  Tao]{yang2024adamerging}
Enneng Yang, Zhenyi Wang, Li~Shen, Shiwei Liu, Guibing Guo, Xingwei Wang, and
  Dacheng Tao.
\newblock {AdaMerging}: Adaptive model merging for multi-task learning.
\newblock In \emph{ICLR}, 2024{\natexlab{a}}.

\bibitem[Yu et~al.(2024)Yu, Yu, Yu, Huang, and Li]{yu2024dare}
Le~Yu, Bowen Yu, Haiyang Yu, Fei Huang, and Yongbin Li.
\newblock Language models are {Super Mario}: Absorbing abilities from
  homologous models as a free lunch.
\newblock In \emph{ICML}, 2024.

\bibitem[Davari and Belilovsky(2024)]{davari2024breadcrumbs}
MohammadReza Davari and Eugene Belilovsky.
\newblock {Model Breadcrumbs}: Scaling multi-task model merging with sparse
  masks.
\newblock In \emph{ECCV}, 2024.

\bibitem[Tam et~al.(2024)Tam, Bansal, and Raffel]{tam2024mats}
Derek Tam, Mohit Bansal, and Colin Raffel.
\newblock Merging by matching models in task parameter subspaces.
\newblock \emph{TMLR}, 2024.

\bibitem[Daheim et~al.(2024)Daheim, M{\"o}llenhoff, Ponti, Gurevych, and
  Khan]{daheim2024gradmatch}
Nico Daheim, Thomas M{\"o}llenhoff, Edoardo~M. Ponti, Iryna Gurevych, and
  Mohammad~Emtiyaz Khan.
\newblock Model merging by uncertainty-based gradient matching.
\newblock In \emph{ICLR}, 2024.

\bibitem[Cheng et~al.(2025)Cheng, Xiong, Wei, Zhu, and Yuan]{cheng2025wudi}
Runxi Cheng, Feng Xiong, Yongxian Wei, Wanyun Zhu, and Chun Yuan.
\newblock Whoever started the interference should end it: Guiding data-free
  model merging via task vectors.
\newblock In \emph{ICML}, 2025.

\bibitem[Yang et~al.(2024{\natexlab{b}})Yang, Shen, Wang, Guo, Chen, Wang, and
  Tao]{yang2024surgery}
Enneng Yang, Li~Shen, Zhenyi Wang, Guibing Guo, Xiaojun Chen, Xingwei Wang, and
  Dacheng Tao.
\newblock Representation surgery for multi-task model merging.
\newblock In \emph{ICML}, 2024{\natexlab{b}}.

\bibitem[Wei et~al.(2025)Wei, He, Yang, Liu, Wang, Feng, and
  An]{wei2025probsurgery}
Qi~Wei, Shuo He, Enneng Yang, Tingcong Liu, Haobo Wang, Lei Feng, and Bo~An.
\newblock Representation surgery in model merging with probabilistic modeling.
\newblock In \emph{ICML}, 2025.

\bibitem[Yang et~al.(2024{\natexlab{c}})Yang, Shen, Wang, Guo, Wang, Cao,
  Zhang, and Tao]{yang2024surgeryv2}
Enneng Yang, Li~Shen, Zhenyi Wang, Guibing Guo, Xingwei Wang, Xiaocun Cao, Jie
  Zhang, and Dacheng Tao.
\newblock {SurgeryV2}: Bridging the gap between model merging and multi-task
  learning with deep representation surgery.
\newblock \emph{arXiv preprint arXiv:2410.14389}, 2024{\natexlab{c}}.

\bibitem[Osial et~al.(2025)Osial, Marczak, and
  Zieli{\'n}ski]{osial2025intervmerge}
Marcin Osial, Daniel Marczak, and Bartosz Zieli{\'n}ski.
\newblock Parameter-efficient interventions for enhanced model merging.
\newblock In \emph{Proceedings of the 2025 SIAM International Conference on
  Data Mining}, 2025.

\bibitem[Ortiz-Jimenez et~al.(2023)Ortiz-Jimenez, Favero, and
  Frossard]{ortizjimenez2023tangenttask}
Guillermo Ortiz-Jimenez, Alessandro Favero, and Pascal Frossard.
\newblock Task arithmetic in the tangent space: Improved editing of pre-trained
  models.
\newblock In \emph{NeurIPS}, 2023.

\bibitem[Jin et~al.(2023)Jin, Ren, Preotiuc-Pietro, and Cheng]{jin2023regmean}
Xisen Jin, Xiang Ren, Daniel Preotiuc-Pietro, and Pengxiang Cheng.
\newblock Dataless knowledge fusion by merging weights of language models.
\newblock In \emph{ICLR}, 2023.

\bibitem[Nguyen et~al.(2025)Nguyen, Dang, Suzuki, and
  Nguyen]{nguyen2025regmeanpp}
The-Hai Nguyen, Huu-Tien Dang, Takeshi Suzuki, and Le-Minh Nguyen.
\newblock {RegMean++}: Enhancing effectiveness and generalization of regression
  mean for model merging.
\newblock \emph{arXiv preprint arXiv:2508.03121}, 2025.

\bibitem[Sun et~al.(2025)Sun, Li, Wang, Liu, Geng, and Li]{sun2025lot}
Wenju Sun, Qingyong Li, Wen Wang, Yang Liu, Yangliao Geng, and Boyang Li.
\newblock Towards minimizing feature drift in model merging: Layer-wise task
  vector fusion for adaptive knowledge integration.
\newblock In \emph{NeurIPS}, 2025.

\bibitem[Hoerl and Kennard(1970)]{hoerl1970ridge}
Arthur~E. Hoerl and Robert~W. Kennard.
\newblock Ridge regression: Biased estimation for nonorthogonal problems.
\newblock \emph{Technometrics}, 1970.

\bibitem[Tikhonov and Arsenin(1977)]{tikhonov1977solutions}
A.~N. Tikhonov and V.~Y. Arsenin.
\newblock \emph{Solutions of Ill-posed Problems}.
\newblock V. H. Winston \& Sons, 1977.
\newblock Distributed solely by Halsted Press.

\bibitem[Tang et~al.(2025)Tang, Shen, Luo, Yang, Hu, Zhang, Du, and
  Tao]{tang2025fusionbench}
Anke Tang, Li~Shen, Yong Luo, Enneng Yang, Han Hu, Lefei Zhang, Bo~Du, and
  Dacheng Tao.
\newblock Fusionbench: A unified library and comprehensive benchmark for deep
  model fusion.
\newblock \emph{JMLR}, 2025.

\bibitem[He et~al.(2025)He, Zeng, Hu, Yang, Zhang, and Zhao]{he2025mergebench}
Yifei He, Siqi Zeng, Yuzheng Hu, Rui Yang, Tong Zhang, and Han Zhao.
\newblock {MergeBench}: A benchmark for merging domain-specialized {LLM}s.
\newblock \emph{arXiv preprint arXiv:2505.10833}, 2025.

\bibitem[Radford et~al.(2021)Radford, Kim, Hallacy, Ramesh, Goh, Agarwal,
  Sastry, Askell, Mishkin, Clark, Krueger, and Sutskever]{radford2021learning}
Alec Radford, Jong~Wook Kim, Chris Hallacy, Aditya Ramesh, Gabriel Goh,
  Sandhini Agarwal, Girish Sastry, Amanda Askell, Pamela Mishkin, Jack Clark,
  Gretchen Krueger, and Ilya Sutskever.
\newblock Learning transferable visual models from natural language
  supervision.
\newblock In \emph{ICML}, 2021.

\bibitem[Xiao et~al.(2010)Xiao, Hays, Ehinger, Oliva, and
  Torralba]{xiao2010sun}
Jianxiong Xiao, James Hays, Krista~A. Ehinger, Aude Oliva, and Antonio
  Torralba.
\newblock {SUN} database: Large-scale scene recognition from abbey to zoo.
\newblock In \emph{CVPR}, 2010.

\bibitem[Krause et~al.(2013)Krause, Stark, Deng, and Fei-Fei]{krause2013cars}
Jonathan Krause, Michael Stark, Jia Deng, and Li~Fei-Fei.
\newblock {3D} object representations for fine-grained categorization.
\newblock In \emph{ICCV Workshops}, 2013.

\bibitem[Cheng et~al.(2017)Cheng, Han, and Lu]{cheng2017remote}
Gong Cheng, Junwei Han, and Xiaoqiang Lu.
\newblock Remote sensing image scene classification: Benchmark and state of the
  art.
\newblock \emph{Proc. IEEE}, 2017.

\bibitem[Helber et~al.(2019)Helber, Bischke, Dengel, and
  Borth]{helber2019eurosat}
Patrick Helber, Benjamin Bischke, Andreas Dengel, and Damian Borth.
\newblock {EuroSAT}: A novel dataset and deep learning benchmark for land use
  and land cover classification.
\newblock \emph{IEEE Journal of Selected Topics in Applied Earth Observations
  and Remote Sensing}, 2019.

\bibitem[Netzer et~al.(2011)Netzer, Wang, Coates, Bissacco, Wu, and
  Ng]{netzer2011reading}
Yuval Netzer, Tao Wang, Adam Coates, Alessandro Bissacco, Bo~Wu, and Andrew~Y.
  Ng.
\newblock Reading digits in natural images with unsupervised feature learning.
\newblock In \emph{NIPS Workshop on Deep Learning and Unsupervised Feature
  Learning}, 2011.

\bibitem[Stallkamp et~al.(2011)Stallkamp, Schlipsing, Salmen, and
  Igel]{stallkamp2011german}
Johannes Stallkamp, Marc Schlipsing, Jan Salmen, and Christian Igel.
\newblock The german traffic sign recognition benchmark: A multi-class
  classification competition.
\newblock In \emph{The 2011 International Joint Conference on Neural Networks},
  2011.

\bibitem[LeCun et~al.(1998)LeCun, Bottou, Bengio, and
  Haffner]{lecun1998gradient}
Yann LeCun, Leon Bottou, Yoshua Bengio, and Patrick Haffner.
\newblock Gradient-based learning applied to document recognition.
\newblock \emph{Proc. IEEE}, 1998.

\bibitem[Cimpoi et~al.(2014)Cimpoi, Maji, Kokkinos, Mohamed, and
  Vedaldi]{cimpoi2014describing}
Mircea Cimpoi, Subhransu Maji, Iasonas Kokkinos, Sammy Mohamed, and Andrea
  Vedaldi.
\newblock Describing textures in the wild.
\newblock In \emph{CVPR}, 2014.

\bibitem[Nilsback and Zisserman(2008)]{nilsback2008automated}
Maria-Elena Nilsback and Andrew Zisserman.
\newblock Automated flower classification over a large number of classes.
\newblock In \emph{Indian Conference on Computer Vision, Graphics and Image
  Processing}, 2008.

\bibitem[Veeling et~al.(2018)Veeling, Linmans, Winkens, Cohen, and
  Welling]{veeling2018rotation}
Bastiaan~S. Veeling, Jasper Linmans, Jim Winkens, Taco Cohen, and Max Welling.
\newblock Rotation equivariant {CNNs} for digital pathology.
\newblock In \emph{Medical Image Computing and Computer Assisted Intervention},
  2018.

\bibitem[Goodfellow et~al.(2015)Goodfellow, Erhan, Carrier, Courville, Mirza,
  Hamner, Cukierski, Tang, Thaler, Lee, Zhou, Ramaiah, Feng, Li, Wang,
  Athanasakis, Shawe-Taylor, Milakov, Park, Ionescu, Popescu, Grozea, Bergstra,
  Xie, Romaszko, Xu, Chuang, and Bengio]{goodfellow2015challenges}
Ian~J. Goodfellow, Dumitru Erhan, Pierre~Luc Carrier, Aaron Courville, Mehdi
  Mirza, Ben Hamner, Will Cukierski, Yichuan Tang, David Thaler, Dong-Hyun Lee,
  Yingbo Zhou, Chetan Ramaiah, Fangxiang Feng, Ruifan Li, Xiaojie Wang,
  Dimitris Athanasakis, John Shawe-Taylor, Maxim Milakov, John Park, Radu
  Ionescu, Marius Popescu, Cristian Grozea, James Bergstra, Jingjing Xie,
  Lukasz Romaszko, Bing Xu, Zhang Chuang, and Yoshua Bengio.
\newblock Challenges in representation learning: A report on three machine
  learning contests.
\newblock \emph{Neural Networks}, 2015.

\bibitem[Parkhi et~al.(2012)Parkhi, Vedaldi, Zisserman, and
  Jawahar]{parkhi2012cats}
Omkar~M. Parkhi, Andrea Vedaldi, Andrew Zisserman, and C.~V. Jawahar.
\newblock Cats and dogs.
\newblock In \emph{CVPR}, 2012.

\bibitem[Coates et~al.(2011)Coates, Ng, and Lee]{coates2011analysis}
Adam Coates, Andrew~Y. Ng, and Honglak Lee.
\newblock An analysis of single-layer networks in unsupervised feature
  learning.
\newblock In \emph{Proceedings of the Fourteenth International Conference on
  Artificial Intelligence and Statistics}, 2011.

\bibitem[Krizhevsky(2009)]{krizhevsky2009learning}
Alex Krizhevsky.
\newblock Learning multiple layers of features from tiny images.
\newblock Technical report, University of Toronto, 2009.

\bibitem[Bossard et~al.(2014)Bossard, Guillaumin, and
  Van~Gool]{bossard2014food}
Lukas Bossard, Matthieu Guillaumin, and Luc Van~Gool.
\newblock {Food-101}: Mining discriminative components with random forests.
\newblock In \emph{ECCV}, 2014.

\bibitem[Xiao et~al.(2017)Xiao, Rasul, and Vollgraf]{xiao2017fashion}
Han Xiao, Kashif Rasul, and Roland Vollgraf.
\newblock {Fashion-MNIST}: A novel image dataset for benchmarking machine
  learning algorithms.
\newblock \emph{arXiv preprint arXiv:1708.07747}, 2017.

\bibitem[Cohen et~al.(2017)Cohen, Afshar, Tapson, and van
  Schaik]{cohen2017emnist}
Gregory Cohen, Saeed Afshar, Jonathan Tapson, and Andre van Schaik.
\newblock {EMNIST}: Extending {MNIST} to handwritten letters.
\newblock In \emph{International Joint Conference on Neural Networks}, 2017.

\bibitem[Clanuwat et~al.(2018)Clanuwat, Bober-Irizar, Kitamoto, Lamb, Yamamoto,
  and Ha]{clanuwat2018deep}
Tarin Clanuwat, Mikel Bober-Irizar, Asanobu Kitamoto, Alex Lamb, Kazuaki
  Yamamoto, and David Ha.
\newblock Deep learning for classical japanese literature.
\newblock In \emph{NeurIPS Workshop on Machine Learning for Creativity and
  Design}, 2018.

\bibitem[Socher et~al.(2013)Socher, Perelygin, Wu, Chuang, Manning, Ng, and
  Potts]{socher2013recursive}
Richard Socher, Alex Perelygin, Jean Wu, Jason Chuang, Christopher~D. Manning,
  Andrew~Y. Ng, and Christopher Potts.
\newblock Recursive deep models for semantic compositionality over a sentiment
  treebank.
\newblock In \emph{EMNLP}, 2013.

\bibitem[{OpenAI}(2021)]{openai2021renderedsst2}
{OpenAI}.
\newblock {Rendered SST-2 Dataset}.
\newblock \url{https://github.com/openai/CLIP/blob/main/data/rendered-sst2.md},
  2021.

\bibitem[Raffel et~al.(2020)Raffel, Shazeer, Roberts, Lee, Narang, Matena,
  Zhou, Li, and Liu]{raffel2020t5}
Colin Raffel, Noam Shazeer, Adam Roberts, Katherine Lee, Sharan Narang, Michael
  Matena, Yanqi Zhou, Wei Li, and Peter~J. Liu.
\newblock Exploring the limits of transfer learning with a unified text-to-text
  transformer.
\newblock \emph{JMLR}, 2020.

\bibitem[Wei et~al.(2022)Wei, Bosma, Zhao, Guu, Yu, Lester, Du, Dai, and
  Le]{wei2022finetuned}
Jason Wei, Maarten Bosma, Vincent~Y. Zhao, Kelvin Guu, Adams~Wei Yu, Brian
  Lester, Nan Du, Andrew~M. Dai, and Quoc~V. Le.
\newblock Finetuned language models are zero-shot learners.
\newblock In \emph{ICLR}, 2022.

\bibitem[Chung et~al.(2024)Chung, Hou, Longpre, Zoph, Tay, Fedus, Li, Wang,
  Dehghani, Brahma, Webson, Gu, Dai, Suzgun, Chen, Chowdhery, Castro-Ros,
  Pellat, Robinson, Valter, Narang, Mishra, Yu, Zhao, Huang, Dai, Yu, Petrov,
  Chi, Dean, Devlin, Roberts, Zhou, Le, and Wei]{chung2024scaling}
Hyung~Won Chung, Le~Hou, Shayne Longpre, Barret Zoph, Yi~Tay, William Fedus,
  Yunxuan Li, Xuezhi Wang, Mostafa Dehghani, Siddhartha Brahma, Albert Webson,
  Shixiang~Shane Gu, Zhuyun Dai, Mirac Suzgun, Xinyun Chen, Aakanksha
  Chowdhery, Alex Castro-Ros, Marie Pellat, Kevin Robinson, Dasha Valter,
  Sharan Narang, Gaurav Mishra, Adams Yu, Vincent Zhao, Yanping Huang,
  Andrew~M. Dai, Hongkun Yu, Slav Petrov, Ed~H. Chi, Jeff Dean, Jacob Devlin,
  Adam Roberts, Denny Zhou, Quoc~V. Le, and Jason Wei.
\newblock Scaling instruction-finetuned language models.
\newblock \emph{JMLR}, 2024.

\bibitem[Wang et~al.(2018)Wang, Singh, Michael, Hill, Levy, and
  Bowman]{wang2018glue}
Alex Wang, Amanpreet Singh, Julian Michael, Felix Hill, Omer Levy, and
  Samuel~R. Bowman.
\newblock {GLUE}: A multi-task benchmark and analysis platform for natural
  language understanding.
\newblock In \emph{Proceedings of the 2018 EMNLP Workshop BlackboxNLP:
  Analyzing and Interpreting Neural Networks for NLP}, 2018.

\bibitem[Warstadt et~al.(2019)Warstadt, Singh, and Bowman]{warstadt2019neural}
Alex Warstadt, Amanpreet Singh, and Samuel~R. Bowman.
\newblock Neural network acceptability judgments.
\newblock \emph{TACL}, 2019.

\bibitem[Williams et~al.(2018)Williams, Nangia, and Bowman]{williams2018broad}
Adina Williams, Nikita Nangia, and Samuel~R. Bowman.
\newblock A broad-coverage challenge corpus for sentence understanding through
  inference.
\newblock In \emph{NAACL-HLT}, 2018.

\bibitem[Dolan and Brockett(2005)]{dolan2005automatically}
William~B. Dolan and Chris Brockett.
\newblock Automatically constructing a corpus of sentential paraphrases.
\newblock In \emph{Proceedings of the Third International Workshop on
  Paraphrasing}, 2005.

\bibitem[Rajpurkar et~al.(2016)Rajpurkar, Zhang, Lopyrev, and
  Liang]{rajpurkar2016squad}
Pranav Rajpurkar, Jian Zhang, Konstantin Lopyrev, and Percy Liang.
\newblock {SQuAD}: 100,000+ questions for machine comprehension of text.
\newblock In \emph{EMNLP}, 2016.

\bibitem[Dagan et~al.(2006)Dagan, Glickman, and Magnini]{dagan2006pascal}
Ido Dagan, Oren Glickman, and Bernardo Magnini.
\newblock The {PASCAL} recognising textual entailment challenge.
\newblock In \emph{Machine Learning Challenges. Evaluating Predictive
  Uncertainty, Visual Object Classification, and Recognising Textual
  Entailment}. Springer, 2006.

\bibitem[Cer et~al.(2017)Cer, Diab, Agirre, Lopez-Gazpio, and
  Specia]{cer2017semeval}
Daniel Cer, Mona Diab, Eneko Agirre, I{\~n}igo Lopez-Gazpio, and Lucia Specia.
\newblock {SemEval}-2017 task 1: Semantic textual similarity multilingual and
  crosslingual focused evaluation.
\newblock In \emph{SemEval}, 2017.

\bibitem[Hu et~al.(2022)Hu, Shen, Wallis, Allen-Zhu, Li, Wang, Wang, and
  Chen]{hu2022lora}
Edward~J. Hu, Yelong Shen, Phillip Wallis, Zeyuan Allen-Zhu, Yuanzhi Li, Shean
  Wang, Lu~Wang, and Weizhu Chen.
\newblock {LoRA}: Low-rank adaptation of large language models.
\newblock In \emph{ICLR}, 2022.

\bibitem[Zhou et~al.(2026)Zhou, Zhang, Gu, Wang, Yan, Li, Chung, and
  Yang]{zhou2026model}
Qi~Zhou, Yiming Zhang, Yanggan Gu, Yuanyi Wang, Zhaoyi Yan, Zhen Li, Chi~Yung
  Chung, and Hongxia Yang.
\newblock Model fusion for scalable and sustainable artificial intelligence: A
  review and outlook.
\newblock \emph{Journal of Modern Power Systems and Clean Energy}, 2026.

\bibitem[Zhou et~al.(2025)Zhou, Zhang, Gu, Wang, Sang, Yan, Li, Zhang, Wu, and
  Yang]{zhou2025democratizing}
Qi~Zhou, Yiming Zhang, Yanggan Gu, Yuanyi Wang, Zhijie Sang, Zhaoyi Yan, Zhen
  Li, Shengyu Zhang, Fei Wu, and Hongxia Yang.
\newblock Democratizing ai through model fusion: A comprehensive review and
  future directions.
\newblock \emph{Nexus}, 2025.

\bibitem[Wang et~al.(2025{\natexlab{a}})Wang, Gu, Zhang, Zhou, Yan, Xie, Wang,
  Yuan, and Yang]{wang2025model}
Yuanyi Wang, Yanggan Gu, Yiming Zhang, Qi~Zhou, Zhaoyi Yan, Congkai Xie, Xinyao
  Wang, Jianbo Yuan, and Hongxia Yang.
\newblock Model merging scaling laws in large language models.
\newblock \emph{arXiv preprint arXiv:2509.24244}, 2025{\natexlab{a}}.

\bibitem[Wang et~al.(2026)Wang, Gu, Wang, Li, Yang, Yan, Xie, Wu, and
  Yang]{wang2026mergepipe}
Yuanyi Wang, Yanggan Gu, Zihao Wang, Kunxi Li, Yifan Yang, Zhaoyi Yan, Congkai
  Xie, Jianmin Wu, and Hongxia Yang.
\newblock {MergePipe}: A budget-aware parameter management system for scalable
  {LLM} merging.
\newblock \emph{arXiv preprint arXiv:2602.13273}, 2026.

\bibitem[Gu et~al.(2025{\natexlab{a}})Gu, Wang, Yan, Zhang, Zhou, Wu, and
  Yang]{gu2025infifpo}
Yanggan Gu, Yuanyi Wang, Zhaoyi Yan, Yiming Zhang, Qi~Zhou, Fei Wu, and Hongxia
  Yang.
\newblock {InfiFPO}: Implicit model fusion via preference optimization in large
  language models.
\newblock \emph{arXiv preprint arXiv:2505.13878}, 2025{\natexlab{a}}.

\bibitem[Gu et~al.(2025{\natexlab{b}})Gu, Li, Huang, Zou, Li, and
  Hu]{gu-etal-2025-capturing}
Yanggan Gu, Junzhuo Li, Sirui Huang, Xin Zou, Zhenghua Li, and Xuming Hu.
\newblock Capturing nuanced preferences: Preference-aligned distillation for
  small language models.
\newblock In \emph{Findings of ACL}, 2025{\natexlab{b}}.

\bibitem[Wang et~al.(2025{\natexlab{b}})Wang, Yan, Zhang, Zhou, Gu, Wu, and
  Yang]{wang2025infigfusion}
Yuanyi Wang, Zhaoyi Yan, Yiming Zhang, Qi~Zhou, Yanggan Gu, Fei Wu, and Hongxia
  Yang.
\newblock {InfiGFusion}: Graph-on-logits distillation via efficient
  gromov-wasserstein for model fusion.
\newblock \emph{arXiv preprint arXiv:2505.13893}, 2025{\natexlab{b}}.

\bibitem[Dang et~al.(2025)Dang, Gao, Yan, Zou, Gu, Li, Wang, Jiang, Liu, Liu,
  and Hu]{dang-etal-2025-exploring}
Yunkai Dang, Mengxi Gao, Yibo Yan, Xin Zou, Yanggan Gu, Jungang Li, Jingyu
  Wang, Peijie Jiang, Aiwei Liu, Jia Liu, and Xuming Hu.
\newblock Exploring response uncertainty in {MLLM}s: An empirical evaluation
  under misleading scenarios.
\newblock In \emph{EMNLP}, 2025.

\bibitem[Yang et~al.(2026)Yang, Li, Li, Zheng, Wang, Qu, Yu, Wu, Li, and
  Yang]{yang2026inficoevalchain}
Yifan Yang, Jinjia Li, Kunxi Li, Puhao Zheng, Yuanyi Wang, Zheyan Qu, Yang Yu,
  Jianmin Wu, Ming Li, and Hongxia Yang.
\newblock {InfiCoEvalChain}: A blockchain-based decentralized framework for
  collaborative {LLM} evaluation.
\newblock \emph{arXiv preprint arXiv:2602.08229}, 2026.

\bibitem[Wang et~al.(2025{\natexlab{c}})Wang, Cai, Xie, Feng, Zhang, Li, Yang,
  Li, Cao, and Yang]{wang2025infir2comprehensivefp8training}
Wenjun Wang, Shuo Cai, Congkai Xie, Mingfa Feng, Yiming Zhang, Zhen Li, Kejing
  Yang, Ming Li, Jiannong Cao, and Hongxia Yang.
\newblock {InfiR2}: A comprehensive {FP8} training recipe for
  reasoning-enhanced language models.
\newblock \emph{arXiv preprint arXiv:2509.22536}, 2025{\natexlab{c}}.

\end{thebibliography}

\clearpage
\appendix
\crefalias{section}{appendix}

\section{Limitations}
\label{app:limitations}

\paragraph{Computational scope.}
Due to limited compute, our experiments focus on CLIP and FLAN-T5.
We do not scale the study to additional modalities or much larger models.
Still, the current experiments cover more than 20 tasks across vision and language, which provides evidence that \method is not tied to a single task or benchmark.

\paragraph{Hyperparameter selection.}
\method still requires hyperparameters, including \(\lambda\), \(\rho\), and \(\alpha\), to be selected on a development or validation set.
We do not provide an automatic selection rule.
Because calibration is fast and uses closed form updates, this search is practical in our setting and adds little overhead compared with iterative calibration baselines.

\paragraph{Dependence on task calibration data.}
Like Surgery and ProbSurgery, \method needs task calibration data to fit parameter updates after merging.
This requirement may be limiting when task data cannot be stored or sampled.
Our results show better data efficiency than these baselines: on CLIP-ViT-B/32 TA, 8 examples per task already give strong gains, and performance saturates with fewer samples than the baselines.

\section{Additional Feature-Distribution Diagnostics for 8-Task TA}
\label{app:ta8-feature-distribution-diagnostics}

\Cref{fig:ta8-feature-distribution-scatter} extends the Stanford Cars visualization in panel (a) of \cref{fig:intro-triptych} to all 8 tasks in the CLIP-ViT-B/32 TA setting.
Each task uses its own joint 2D projection, so the geometry is intended for within-task comparison between the uncalibrated and calibrated merged features rather than cross-task distance comparison.

\begin{figure}[H]
\centering
\includegraphics[width=\textwidth]{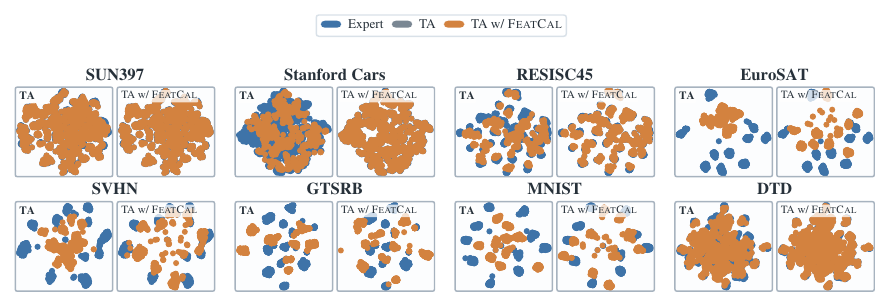}
\caption{
\textbf{8-task feature-distribution diagnostic.}
Matched final projected features for the task expert, Task Arithmetic (TA), and TA w/ \method{} on each task in the CLIP-ViT-B/32 TA setting, following panel (a) of \cref{fig:intro-triptych}.
}
\label{fig:ta8-feature-distribution-scatter}
\end{figure}

\section{Proofs for the Feature-Drift Analysis}
\label{app:proof-feature-drift}

\begin{proof}[Proof of \cref{prop:exact-single-layer-decomposition}]
Starting from the definition of \(e_{i,\ell}(x)\), add and subtract
\(
f_{\ell}^{\merge}(h_{i,\ell-1}^{\expert}(x))
\)
to obtain
\begin{align}
e_{i,\ell}(x)
&=
f_{\ell}^{\merge}\!\left(h_{i,\ell-1}^{\merge}(x)\right)
-
f_{i,\ell}^{\expert}\!\left(h_{i,\ell-1}^{\expert}(x)\right)
\notag\\
&=
\Bigl[
f_{\ell}^{\merge}\!\left(h_{i,\ell-1}^{\merge}(x)\right)
-
f_{\ell}^{\merge}\!\left(h_{i,\ell-1}^{\expert}(x)\right)
\Bigr]
\notag\\
&\quad+
\Bigl[
f_{\ell}^{\merge}\!\left(h_{i,\ell-1}^{\expert}(x)\right)
-
f_{i,\ell}^{\expert}\!\left(h_{i,\ell-1}^{\expert}(x)\right)
\Bigr].
\end{align}
The first bracket is \(p_{i,\ell}(x)\) and the second is \(m_{i,\ell}(x)\), proving \cref{eq:exact-single-layer-decomposition}.
\end{proof}

\begin{proof}[Proof of \cref{prop:layerwise-propagation-identity}]
For a fixed layer \(\ell\), the assumed differentiability on a neighborhood of the expert-to-merged feature segment lets us define the averaged Jacobian
\begin{equation}
\label{eq:avg-jacobian-general}
A_{i,\ell}(x)
=
\int_{0}^{1}
J_u f_{\ell}^{\merge}\!\Bigl(
h_{i,\ell-1}^{\expert}(x)
+
t\,e_{i,\ell-1}(x)
\Bigr)\,dt.
\end{equation}
The fundamental theorem of calculus along this segment gives
\begin{equation}
\label{eq:propagated-component-jacobian}
p_{i,\ell}(x)
=
A_{i,\ell}(x)\,e_{i,\ell-1}(x),
\end{equation}
and hence
\begin{equation}
\label{eq:general-recursion}
e_{i,\ell}(x)
=
A_{i,\ell}(x)\,e_{i,\ell-1}(x)
+
m_{i,\ell}(x).
\end{equation}
Iterating \cref{eq:general-recursion} from layer \(s+1\) to layer \(t\) yields
\begin{align}
\label{eq:general-layerwise-unroll}
e_{i,t}(x)
&=
P_{i,s\rightarrow t}(x)e_{i,s}(x)
+
\sum_{\ell=s+1}^{t}
P_{i,\ell\rightarrow t}(x)m_{i,\ell}(x),
\end{align}
where \(P_{i,r\rightarrow t}(x)=A_{i,t}(x)\cdots A_{i,r+1}(x)\), \(P_{i,t\rightarrow t}(x)=I\), and products denote composition of compatible local linear maps. Taking \(s=0\), \(t=L\), and \(e_{i,0}(x)=0\) yields
\begin{equation}
\label{eq:unrolled-accumulation}
e_{i,L}(x)
=
\sum_{\ell=1}^{L}
P_{i,\ell\rightarrow L}(x)\,m_{i,\ell}(x),
\end{equation}
where
\begin{equation}
\label{eq:phi-definition}
P_{i,\ell\rightarrow L}(x)
=
A_{i,L}(x)A_{i,L-1}(x)\cdots A_{i,\ell+1}(x),
\qquad
P_{i,L\rightarrow L}(x)=I.
\end{equation}
Thus every term in the final drift expansion originates from a layer-wise local mismatch and is propagated through the downstream sensitivity chain.
\end{proof}

\section{Residual Propagation and Conditional Growth}
\label{app:residual-growth}

This appendix specializes \cref{prop:exact-single-layer-decomposition} to residualized Transformer layers with matching input and output dimensions:
\begin{equation}
\label{eq:residual-block-form}
f_\ell^{\merge}(h)=h+g_\ell^{\merge}(h),
\qquad
f_{i,\ell}^{\expert}(h)=h+g_{i,\ell}^{\expert}(h),
\end{equation}
where \(g_\ell^{\merge}\) and \(g_{i,\ell}^{\expert}\) collect module-internal nonlinear operations, such as self-attention or MLP together with normalization and other internal operations.

Define the residual-branch propagation term
\begin{equation}
\label{eq:residual-branch-propagation}
r_{i,\ell}(x)
=
g_\ell^{\merge}\!\left(h_{i,\ell-1}^{\merge}(x)\right)
-
g_\ell^{\merge}\!\left(h_{i,\ell-1}^{\expert}(x)\right),
\end{equation}
and the local mismatch term becomes
\begin{equation}
\label{eq:residual-local-mismatch}
m_{i,\ell}(x)
=
g_\ell^{\merge}\!\left(h_{i,\ell-1}^{\expert}(x)\right)
-
g_{i,\ell}^{\expert}\!\left(h_{i,\ell-1}^{\expert}(x)\right).
\end{equation}

\begin{proposition}[Residual preservation]
\label{prop:residual-preservation}
For a residualized layer of the form in \cref{eq:residual-block-form}, the propagated component satisfies
\begin{equation}
\label{eq:residual-propagated-component}
p_{i,\ell}(x)
=
e_{i,\ell-1}(x)
+
r_{i,\ell}(x),
\end{equation}
and therefore
\begin{equation}
\label{eq:residual-decomposition}
e_{i,\ell}(x)
=
e_{i,\ell-1}(x)
+
r_{i,\ell}(x)
+
m_{i,\ell}(x).
\end{equation}
\end{proposition}

\begin{proof}
By \cref{eq:residual-block-form},
\begin{align}
p_{i,\ell}(x)
&=
f_{\ell}^{\merge}\!\left(h_{i,\ell-1}^{\merge}(x)\right)
-
f_{\ell}^{\merge}\!\left(h_{i,\ell-1}^{\expert}(x)\right)
\notag\\
&=
\Bigl(
h_{i,\ell-1}^{\merge}(x)
+
g_\ell^{\merge}\!\left(h_{i,\ell-1}^{\merge}(x)\right)
\Bigr)
-
\Bigl(
h_{i,\ell-1}^{\expert}(x)
+
g_\ell^{\merge}\!\left(h_{i,\ell-1}^{\expert}(x)\right)
\Bigr)
\notag\\
&=
\underbrace{
\bigl(
h_{i,\ell-1}^{\merge}(x)-h_{i,\ell-1}^{\expert}(x)
\bigr)
}_{e_{i,\ell-1}(x)}
+
\underbrace{
\Bigl(
g_\ell^{\merge}\!\left(h_{i,\ell-1}^{\merge}(x)\right)
-
g_\ell^{\merge}\!\left(h_{i,\ell-1}^{\expert}(x)\right)
\Bigr)
}_{r_{i,\ell}(x)},
\end{align}
which proves \cref{eq:residual-propagated-component}. Substituting this identity into \cref{eq:exact-single-layer-decomposition} gives \cref{eq:residual-decomposition}.
\end{proof}

\paragraph{Interpretation.}
\cref{eq:residual-decomposition} makes the identity-skip propagation property explicit: the skip connection carries upstream feature drift through the additive term \(e_{i,\ell-1}(x)\), while the residual branch may compensate for or amplify it.

When \(g_\ell^{\merge}(\cdot)\) is continuously differentiable on an open neighborhood containing the line segment between \(h_{i,\ell-1}^{\expert}(x)\) and \(h_{i,\ell-1}^{\merge}(x)\), define
\begin{equation}
\label{eq:residual-avg-jacobian}
R_{i,\ell}(x)
=
\int_0^1
J_u g_\ell^{\merge}\!\Bigl(
h_{i,\ell-1}^{\expert}(x)+t\,e_{i,\ell-1}(x)
\Bigr)\,dt.
\end{equation}
Then
\begin{equation}
\label{eq:residual-linearized-recursion}
r_{i,\ell}(x)=R_{i,\ell}(x)\,e_{i,\ell-1}(x),
\end{equation}
and \cref{eq:residual-decomposition} becomes
\begin{equation}
\label{eq:residual-jacobian-recursion}
e_{i,\ell}(x)
=
\bigl(I+R_{i,\ell}(x)\bigr)e_{i,\ell-1}(x)
+
m_{i,\ell}(x).
\end{equation}

\begin{proposition}[Sufficient condition for local monotone amplification]
\label{prop:local-error-growth}
Consider a residualized layer of the form in \cref{eq:residual-block-form},
and assume \(g_\ell^{\merge}\) is continuously differentiable on an open
neighborhood containing the line segment between
\(h_{i,\ell-1}^{\expert}(x)\) and \(h_{i,\ell-1}^{\merge}(x)\), so
\(R_{i,\ell}(x)\) is well defined. Suppose \(e_{i,\ell-1}(x)\neq 0\).
Suppose further that there exists \(\gamma_{i,\ell}>0\) such that, for \(v=e_{i,\ell-1}(x)\),
\begin{equation}
\label{eq:expansion-condition}
\bigl\|\bigl(I+R_{i,\ell}(x)\bigr)v\bigr\|_2
\ge
(1+\gamma_{i,\ell})\|v\|_2,
\end{equation}
and suppose the local mismatch term satisfies
\begin{equation}
\label{eq:local-mismatch-bound}
\|m_{i,\ell}(x)\|_2
\le
\eta_{i,\ell}\,\|e_{i,\ell-1}(x)\|_2
\qquad
\text{for some } \eta_{i,\ell}\in[0,\gamma_{i,\ell}).
\end{equation}
Then
\begin{equation}
\label{eq:local-growth-bound}
\|e_{i,\ell}(x)\|_2
\ge
(1+\gamma_{i,\ell}-\eta_{i,\ell})
\|e_{i,\ell-1}(x)\|_2
>
\|e_{i,\ell-1}(x)\|_2.
\end{equation}
\end{proposition}

\begin{proof}
By \cref{eq:residual-jacobian-recursion} and the reverse triangle inequality,
\begin{align}
\|e_{i,\ell}(x)\|_2
&=
\left\|
\bigl(I+R_{i,\ell}(x)\bigr)e_{i,\ell-1}(x)
+
m_{i,\ell}(x)
\right\|_2 \notag\\
&\ge
\left\|
\bigl(I+R_{i,\ell}(x)\bigr)e_{i,\ell-1}(x)
\right\|_2
-
\|m_{i,\ell}(x)\|_2 \notag\\
&\ge
(1+\gamma_{i,\ell})\|e_{i,\ell-1}(x)\|_2
-
\eta_{i,\ell}\|e_{i,\ell-1}(x)\|_2 \notag\\
&=
(1+\gamma_{i,\ell}-\eta_{i,\ell})
\|e_{i,\ell-1}(x)\|_2.
\end{align}
Since \(e_{i,\ell-1}(x)\neq 0\) and \(\eta_{i,\ell}<\gamma_{i,\ell}\), the multiplicative factor is strictly larger than \(1\), proving the claim.
\end{proof}

\cref{prop:local-error-growth} is intentionally conditional: feature drift need not amplify at every residualized layer, but it does grow monotonically when the residualized layer map is locally expansive along the current drift direction and the new local mismatch term is too small to offset that expansion.

\begin{corollary}[Conditional cumulative monotone amplification]
\label{cor:conditional-cumulative-growth}
If there exists a consecutive range of layers \(\ell_0+1,\dots,L\) such that the assumptions of \cref{prop:local-error-growth} hold for every layer in this range, then
\begin{equation}
\label{eq:cumulative-growth}
\|e_{i,L}(x)\|_2
\ge
\prod_{\ell=\ell_0+1}^{L}
(1+\gamma_{i,\ell}-\eta_{i,\ell})
\,
\|e_{i,\ell_0}(x)\|_2.
\end{equation}
Consequently, the feature drift norm increases strictly from a layer to the
next across that range.
\end{corollary}

\begin{proof}
Apply \cref{eq:local-growth-bound} recursively from layer \(\ell_0+1\) to layer \(L\).
\end{proof}

\section{Output-Drift Bridge}
\label{app:output-drift-bridge}

This appendix records the standard perturbation argument connecting feature drift
to output-score drift in the continuous task output score vector. For the merged
model \(M^{\merge}\), define
\[
e_{i,L}(x)=h_{i,L}^{\merge}(x)-h_{i,L}^{\expert}(x).
\]
Let the task score maps \(\psi_i^{\merge}\) and \(\psi_i^{\expert}\) produce
\[
z_i^{\merge}(x)=\psi_i^{\merge}\!\left(h_{i,L}^{\merge}(x)\right),
\qquad
z_i^{\expert}(x)=\psi_i^{\expert}\!\left(h_{i,L}^{\expert}(x)\right),
\]
and define
\[
\Delta z_i^{\merge}(x)=z_i^{\merge}(x)-z_i^{\expert}(x).
\]
Here \(z_i\) may denote class logits, CLIP candidate scores (scaled similarity
scores over a fixed candidate set), or decoder vocabulary logits under a fixed
prefix.

\begin{proposition}[Output score map perturbation]
\label{prop:score-map-perturbation}
Suppose \(\psi_i^{\merge}\) is locally \(B_i^{\merge}(x)\)-Lipschitz on the segment between \(h_{i,L}^{\expert}(x)\) and \(h_{i,L}^{\merge}(x)\).
Here \(B_i^{\merge}(x)\) locally bounds how much the merged task score map can change when its final-feature input moves along this segment. Then
\begin{align}
\left\|\Delta z_i^{\merge}(x)\right\|_2
&\le
B_i^{\merge}(x)\left\|e_{i,L}(x)\right\|_2
\;+\;
\left\|
\psi_i^{\merge}\!\left(h_{i,L}^{\expert}(x)\right)
-
\psi_i^{\expert}\!\left(h_{i,L}^{\expert}(x)\right)
\right\|_2 .
\label{eq:score-map-bound}
\end{align}
If the task score map is shared, the second term is \(0\).
\end{proposition}

\begin{proof}
Add and subtract \(\psi_i^{\merge}(h_{i,L}^{\expert}(x))\), then apply the local Lipschitz bound to the first difference:
\begin{align}
\left\|\Delta z_i^{\merge}(x)\right\|_2
&=
\left\|
\psi_i^{\merge}\!\left(h_{i,L}^{\merge}(x)\right)
-
\psi_i^{\expert}\!\left(h_{i,L}^{\expert}(x)\right)
\right\|_2
\notag\\
&\le
\left\|
\psi_i^{\merge}\!\left(h_{i,L}^{\merge}(x)\right)
-
\psi_i^{\merge}\!\left(h_{i,L}^{\expert}(x)\right)
\right\|_2
\notag\\
&\quad+
\left\|
\psi_i^{\merge}\!\left(h_{i,L}^{\expert}(x)\right)
-
\psi_i^{\expert}\!\left(h_{i,L}^{\expert}(x)\right)
\right\|_2
\notag\\
&\le
B_i^{\merge}(x)\left\|e_{i,L}(x)\right\|_2
\;+\;
\left\|
\psi_i^{\merge}\!\left(h_{i,L}^{\expert}(x)\right)
-
\psi_i^{\expert}\!\left(h_{i,L}^{\expert}(x)\right)
\right\|_2 .
\end{align}
\end{proof}

\paragraph{Softmax probability bridge.}
For a softmax output over a fixed candidate or token set, let
\(p(z)=\operatorname{softmax}(z)\). Its Jacobian is
\[
J_{\operatorname{sm}}(z)
=
\nabla_z p(z)
=
\operatorname{diag}(p(z))-p(z)p(z)^\top .
\]
Hence, for any logit perturbation \(\Delta z\),
\[
p(z+\Delta z)-p(z)
=
\int_0^1
J_{\operatorname{sm}}(z+t\Delta z)\Delta z\,dt .
\]
For cross-entropy with one-hot label vector \(u_y\),
\[
\mathcal{L}_{\mathrm{CE}}(y,z+\Delta z)-\mathcal{L}_{\mathrm{CE}}(y,z)
=
\int_0^1
\bigl(p(z+t\Delta z)-u_y\bigr)^\top\Delta z\,dt .
\]
Thus relative logit changes induce softmax-normalized probability and loss
changes through the softmax Jacobian and the cross-entropy gradient; an
additive shift \(c\mathbf{1}\) is a null direction for both.

\begin{proposition}[Logit perturbation and loss stability]
\label{prop:score-loss-stability}
Let \(\mathcal{L}_i(y,z)\) be a task loss that is differentiable along the segment between \(z_i^{\expert}(x)\) and \(z_i^{\merge}(x)\). Define
\[
\bar g_i(x)=
\int_0^1
\nabla_z \mathcal{L}_i\!\left(y,z_i^{\expert}(x)+t\Delta z_i^{\merge}(x)\right)\,dt .
\]
Then
\begin{equation}
\label{eq:score-loss-identity}
\mathcal{L}_i(y,z_i^{\merge}(x))-\mathcal{L}_i(y,z_i^{\expert}(x))
=
\bar g_i(x)^\top \Delta z_i^{\merge}(x),
\end{equation}
and hence
\begin{equation}
\label{eq:score-loss-bound}
\left|
\mathcal{L}_i(y,z_i^{\merge}(x))-\mathcal{L}_i(y,z_i^{\expert}(x))
\right|
\le
\left\|\bar g_i(x)\right\|_2
\left\|\Delta z_i^{\merge}(x)\right\|_2 .
\end{equation}
\end{proposition}

\begin{proof}
Apply the fundamental theorem of calculus to
\(\mathcal{L}_i(y,z_i^{\expert}(x)+t\Delta z_i^{\merge}(x))\) over \(t\in[0,1]\), and then use Cauchy--Schwarz.
\end{proof}

\begin{corollary}[Feature-to-loss perturbation]
\label{cor:feature-to-loss-perturbation}
Under the conditions of \cref{prop:score-map-perturbation}, suppose
\(\mathcal{L}_i(y,z)\) is locally \(G_i^{\merge}(x)\)-Lipschitz on the segment between
\(z_i^{\expert}(x)\) and \(z_i^{\merge}(x)\), and let \(\delta_i^{\psi,\merge}(x)\) denote the
score map mismatch term in \cref{eq:score-map-bound}. Then
\begin{equation}
\label{eq:feature-to-loss-bound}
\left|
\mathcal{L}_i(y,z_i^{\merge}(x))-\mathcal{L}_i(y,z_i^{\expert}(x))
\right|
\le
G_i^{\merge}(x)
\left(
B_i^{\merge}(x)\left\|e_{i,L}(x)\right\|_2
+
\delta_i^{\psi,\merge}(x)
\right).
\end{equation}
\end{corollary}

\begin{proof}
Combine \cref{eq:score-map-bound} with the local Lipschitz bound on
\(\mathcal{L}_i(y,\cdot)\).
\end{proof}

\paragraph{Decision margins.}
For discrete decisions, output-score drift matters through margins. For a top-1
classification output with unique expert winner
\(k^\star=\arg\max_k z_{i,k}^{\expert}(x)\), each competitor \(k\ne k^\star\)
changes its pairwise margin as
\[
z_{i,k^\star}^{\merge}(x)-z_{i,k}^{\merge}(x)
=
\bigl(z_{i,k^\star}^{\expert}(x)-z_{i,k}^{\expert}(x)\bigr)
-
\bigl(\Delta z_{i,k}^{\merge}(x)-\Delta z_{i,k^\star}^{\merge}(x)\bigr).
\]
Thus a decision boundary is crossed only when a competitor's relative score
drift closes the corresponding expert margin. In particular, with the minimum
expert margin
\[
\mu_i(x)=z_{i,k^\star}^{\expert}(x)-\max_{k\ne k^\star}z_{i,k}^{\expert}(x),
\]
the merged model preserves the expert top-1 decision whenever
\begin{equation}
\label{eq:top1-margin-main}
\|\Delta z_i^{\merge}(x)\|_\infty < \mu_i(x)/2
\quad\Longrightarrow\quad
\arg\max_k z_{i,k}^{\merge}(x)=k^\star .
\end{equation}
Analogous pairwise margins apply to fixed-candidate retrieval items and
fixed-prefix next-token choices.

\section{Forward-Order Calibration and Deployed Feature Sources}
\label{app:forward-order-calibration}

This appendix formalizes why \method collects current calibrated-model features
at each layer before fitting the calibrated modules in that layer. The point is
complementary to \cref{prop:layerwise-propagation-identity}: local mismatch
produced at a layer is propagated downstream in forward order, so the layer-wise
calibration surrogate should be fitted on the feature source that the final
calibrated model will actually expose to that layer.

Let \(M^{[0]}=M^{\merge}\), and let \(M^{[\ell-1]}\) denote the partially
calibrated model after all calibrated layers before layer \(\ell\) have been
updated. Because updates to layer \(\ell\) and later layers are downstream of the
prefix before layer \(\ell\), the input feature to layer \(\ell\) in the final
calibrated model \(M^{\cali}\) is the same as the input feature produced by
\(M^{[\ell-1]}\):
\begin{equation}
\label{eq:deployed-prefix-feature}
h_{i,\ell-1}^{M^{\cali}}(x)
=
h_{i,\ell-1}^{M^{[\ell-1]}}(x).
\end{equation}
Thus the deployed layer-input distribution is the push-forward distribution
\begin{equation}
\label{eq:deployed-prefix-distribution}
\mu_{i,\ell}^{\dep}
=
\bigl(h_{i,\ell-1}^{M^{[\ell-1]}}\bigr)_{\#}\mathcal D_i .
\end{equation}
Within layer \(\ell\), \method uses this cached layer snapshot to construct the
input features for all calibrated modules in the layer. It does not advance the
feature source after each module update inside the same layer. Thus the deployed
source statement is layer-level, while the objectives below are module-local fits
computed from cached features.

Consider a calibrated linear module with candidate weight \(W\). Let
\((u,y)\) denote the module input feature and the expert target output
used by the local regression, and define
\[
r_W(u,y)=\|Wu-y\|_2^2 .
\]
Let \(\Pi_{i,\ell}^{\dep}\) be the deployed joint distribution of these
input-target pairs induced by the prefix-calibrated feature source in
\cref{eq:deployed-prefix-distribution}. Let \(\Pi_{i,\ell}^{\src}\) be any other
joint distribution used to fit the same local update, such as features collected
from the uncalibrated merged model, from the task expert, or from a stale
partially calibrated model. For \(a\in\{\src,\dep\}\), define the population
ridge objective
\begin{equation}
\label{eq:source-deployed-objective}
J_{\ell}^{a}(W)
=
\sum_{i=1}^{N}
\mathbb E_{(u,y)\sim \Pi_{i,\ell}^{a}}
r_W(u,y)
+
\lambda\|W-W^{\anchor}\|_F^2 .
\end{equation}

\begin{proposition}[Source/deployed objective mismatch]
\label{prop:source-deployed-objective-mismatch}
Let \(\mathcal W\) be a weight class for which the following suprema are finite,
and let
\[
W_{\src}^{\star}\in\arg\min_{W\in\mathcal W}J_{\ell}^{\src}(W),
\qquad
W_{\dep}^{\star}\in\arg\min_{W\in\mathcal W}J_{\ell}^{\dep}(W).
\]
Then
\begin{equation}
\label{eq:source-deployed-excess-bound}
J_{\ell}^{\dep}(W_{\src}^{\star})
-
J_{\ell}^{\dep}(W_{\dep}^{\star})
\le
2\sup_{W\in\mathcal W}
\left|
J_{\ell}^{\dep}(W)-J_{\ell}^{\src}(W)
\right|.
\end{equation}
If, moreover, \(r_W(u,y)\) is uniformly \(K_{\ell}\)-Lipschitz in \((u,y)\) on
the relevant supports for all \(W\in\mathcal W\), then
\begin{equation}
\label{eq:source-deployed-wasserstein-bound}
J_{\ell}^{\dep}(W_{\src}^{\star})
-
J_{\ell}^{\dep}(W_{\dep}^{\star})
\le
2K_{\ell}
\sum_{i=1}^{N}
W_1\!\left(
\Pi_{i,\ell}^{\dep},
\Pi_{i,\ell}^{\src}
\right),
\end{equation}
where \(W_1\) denotes the 1-Wasserstein distance on input-target pairs.
\end{proposition}

\begin{proof}
Add and subtract \(J_{\ell}^{\src}(W_{\src}^{\star})\) and
\(J_{\ell}^{\src}(W_{\dep}^{\star})\):
\begin{align}
J_{\ell}^{\dep}(W_{\src}^{\star})
-
J_{\ell}^{\dep}(W_{\dep}^{\star})
&=
\bigl[J_{\ell}^{\dep}(W_{\src}^{\star})
-
J_{\ell}^{\src}(W_{\src}^{\star})\bigr]
\notag\\
&\quad+
\bigl[J_{\ell}^{\src}(W_{\src}^{\star})
-
J_{\ell}^{\src}(W_{\dep}^{\star})\bigr]
\notag\\
&\quad+
\bigl[J_{\ell}^{\src}(W_{\dep}^{\star})
-
J_{\ell}^{\dep}(W_{\dep}^{\star})\bigr].
\end{align}
The middle bracket is nonpositive by optimality of \(W_{\src}^{\star}\) for
\(J_{\ell}^{\src}\). Both remaining brackets are bounded by
\(\sup_{W\in\mathcal W}|J_{\ell}^{\dep}(W)-J_{\ell}^{\src}(W)|\), which proves
\cref{eq:source-deployed-excess-bound}. The quadratic regularization term is identical
in both objectives, so the source/deployed difference only comes from the
expected matching losses. The Kantorovich--Rubinstein duality for \(W_1\), under
the stated uniform Lipschitz condition, gives
\[
\left|
\mathbb E_{\Pi_{i,\ell}^{\dep}}r_W
-
\mathbb E_{\Pi_{i,\ell}^{\src}}r_W
\right|
\le
K_{\ell}
W_1\!\left(\Pi_{i,\ell}^{\dep},\Pi_{i,\ell}^{\src}\right).
\]
Summing over tasks and applying the previous bound proves
\cref{eq:source-deployed-wasserstein-bound}.
\end{proof}

\paragraph{Implication for \method.}
Forward-order calibration sets the source distribution for layer \(\ell\) to the
deployed prefix distribution in \cref{eq:deployed-prefix-distribution}, so the
population source/deployed mismatch term in
\cref{eq:source-deployed-wasserstein-bound} vanishes for that layer snapshot.
Using merged-model, expert, or stale features can fit a different objective:
after earlier layers are calibrated, the final model may expose layer \(\ell\)
to a feature source that differs from the source used to estimate \(W\). This does
not prove global optimality of the forward-order procedure, but it explains why
\method captures features at each layer from the current partially calibrated
model rather than from a non-deployed feature source.

\section{Bias and LayerNorm Affine Updates}
\label{app:affine-layernorm-calibration}

The main text writes the core update for a bias-free linear map. The
implementation also calibrates a linear-module bias parameter when present. It
first solves the weight \(W^\star\) by \cref{eq:featcal-final-closed-form}, then
solves the bias with \(W^\star\) fixed. Thus the following formula is a
second-stage bias update, not a joint affine solve over \(W\) and \(b\).
For readability, the equations below use uniform empirical averages; replacing
these averages by normalized sample or token weights gives the implemented
weighted version.

Let \(b_i,b^{\merge},b^{\base}\in\mathbb{R}^{m}\) be the expert, merged, and
base bias parameters for the same linear module, and define
\begin{equation}
\label{eq:featcal-bias-anchor}
b^{\anchor}
=
\rho b^{\merge}
+
(1-\rho)b^{\base}.
\end{equation}
Let \(\mathbf{1}_n\in\mathbb{R}^{n}\) be the all-one vector and define the empirical
input means
\begin{equation}
\label{eq:featcal-bias-means}
\mu_i^{\cali}
=
\frac{1}{n}X_i^{\cali}\mathbf{1}_n,
\qquad
\mu_i^{\tgt}
=
\frac{1}{n}X_i^{\tgt}\mathbf{1}_n.
\end{equation}
The mean \(\mu_i^{\tgt}\) is taken after target-side interpolation in
\cref{eq:featcal-target-feature}; it equals the raw expert-feature mean only when
\(\alpha=1\). With \(W^\star\) fixed and the task weights \(\omega_i\) from
\cref{eq:featcal-task-weight}, the bias objective is
\begin{equation}
\label{eq:featcal-bias-objective}
b^\star
=
\arg\min_{b\in\mathbb{R}^{m}}
\sum_{i=1}^{N}
\omega_i
\left\|
W^\star\mu_i^{\cali}
+
b
-
\left(W_i\mu_i^{\tgt}+b_i\right)
\right\|_2^2
+
\lambda\|b-b^{\anchor}\|_2^2 .
\end{equation}
Taking the derivative gives
\begin{equation}
\label{eq:featcal-bias-stationary}
\left(\sum_{i=1}^{N}\omega_i+\lambda\right)b
=
\sum_{i=1}^{N}
\omega_i
\left(
b_i
+
W_i\mu_i^{\tgt}
-
W^\star\mu_i^{\cali}
\right)
+
\lambda b^{\anchor},
\end{equation}
and hence
\begin{equation}
\label{eq:featcal-bias-closed-form}
b^\star
=
\frac{
\sum_{i=1}^{N}
\omega_i
\left(
b_i+W_i\mu_i^{\tgt}-W^\star\mu_i^{\cali}
\right)
+
\lambda b^{\anchor}
}{
\sum_{i=1}^{N}\omega_i+\lambda
}.
\end{equation}
The experiments in this paper provide merged and base anchors. If an anchor is
unavailable in code, the corresponding \(\lambda\)-term is omitted from the
numerator and denominator.

For a LayerNorm module, write \(\operatorname{LN}_0\) for LayerNorm without its
affine scale and shift. For task \(i\), \method collects the current calibrated
input features and forms
\[
Z_i^{\cali}
=
\operatorname{LN}_0(X_i^{\cali}),
\]
using the module's normalized shape and its LayerNorm numerical constant. This LayerNorm
update uses the current normalized features only: it matches a shared affine
map to the expert affine maps evaluated on \(Z_i^{\cali}\), rather than matching
expert-normalized features.

Let \(\gamma_i,\beta_i\in\mathbb{R}^{d}\) be the expert LayerNorm scale and
shift for task \(i\), and let
\(\gamma^{\merge},\beta^{\merge},\gamma^{\base},\beta^{\base}\in\mathbb{R}^{d}\)
be the corresponding merged and base affine parameters. Define
\[
\gamma^{\anchor}
=
\rho\gamma^{\merge}
+
(1-\rho)\gamma^{\base},
\qquad
\beta^{\anchor}
=
\rho\beta^{\merge}
+
(1-\rho)\beta^{\base}.
\]
For each task, summarize the current normalized features by the coordinate-wise
mean and second moment
\begin{equation}
\label{eq:featcal-layernorm-moments}
\bar z_i
=
\frac{1}{n}Z_i^{\cali}\mathbf{1}_n,
\qquad
q_i
=
\frac{1}{n}
\left(Z_i^{\cali}\odot Z_i^{\cali}\right)\mathbf{1}_n,
\end{equation}
where vector products with \(Z_i^{\cali}\) are applied column-wise. The
LayerNorm affine objective is
\begin{equation}
\label{eq:featcal-layernorm-objective}
\begin{aligned}
(\gamma^\star,\beta^\star)
=
\arg\min_{\gamma,\beta\in\mathbb{R}^{d}}
&
\sum_{i=1}^{N}
\frac{1}{n}
\left\|
\gamma\odot Z_i^{\cali}
+
\beta\mathbf{1}_n^{\top}
-
\left(
\gamma_i\odot Z_i^{\cali}
+
\beta_i\mathbf{1}_n^{\top}
\right)
\right\|_F^2
\\
&+
\lambda\|\gamma-\gamma^{\anchor}\|_2^2
+
\lambda\|\beta-\beta^{\anchor}\|_2^2 .
\end{aligned}
\end{equation}
Unlike the linear-weight and bias updates, this LayerNorm affine update uses
equal task weights after each task's empirical moments are formed. This objective
separates across coordinates. Let
\(\mathbf{1}_d\in\mathbb{R}^{d}\) be the all-one vector and define
\begin{align}
\label{eq:featcal-layernorm-normal-terms}
a_{11}
&=
\sum_{i=1}^{N}q_i+\lambda\mathbf{1}_d,
&
a_{12}
&=
\sum_{i=1}^{N}\bar z_i,
&
a_{22}
&=
N+\lambda,
\notag\\
r_{\gamma}
&=
\sum_{i=1}^{N}
\left(q_i\odot\gamma_i+\bar z_i\odot\beta_i\right)
+
\lambda\gamma^{\anchor},
&
r_{\beta}
&=
\sum_{i=1}^{N}
\left(\bar z_i\odot\gamma_i+\beta_i\right)
+
\lambda\beta^{\anchor}.
\end{align}
The coordinate-wise \(2\times2\) normal equations have determinant
\[
D
=
a_{22}a_{11}
-
a_{12}\odot a_{12}.
\]
Thus the closed-form affine update is
\begin{equation}
\label{eq:featcal-layernorm-closed-form}
\gamma^\star
=
\left(a_{22}r_{\gamma}-a_{12}\odot r_{\beta}\right)\oslash D,
\qquad
\beta^\star
=
\left(a_{11}\odot r_{\beta}-a_{12}\odot r_{\gamma}\right)\oslash D,
\end{equation}
where \(\oslash\) denotes element-wise division. In the implementation,
near-singular determinants are clamped from below by the FeatCal numerical
stabilizer \(\epsilon\). If an anchor is unavailable, the corresponding
\(\lambda\)-term is omitted.

\section{Calibration Algorithm}
\label{app:calibration-algorithm}

\begin{algorithm}[H]
\caption{\method Calibration}
\label{alg:featcal}
\begin{algorithmic}[1]
\Require \(M^{\merge}\), \(M^{\base}\), experts \(\{M_i^{\expert}\}_{i=1}^N\), calibration sets, layers or blocks in forward order with calibrated modules inside each layer, and hyperparameters \(\lambda,\alpha,\rho,\epsilon\)
\Ensure calibrated model \(M^{\cali}\)
\State Initialize \(M^{\cali}\leftarrow M^{\merge}\)
\For{each layer or block in forward order}
    \State Cache current calibrated-model and expert input features once for the modules calibrated in this layer
    \For{each module \(t\) calibrated in the current layer}
    \If{\(t\) is a linear module}
    \State Let \(W^{\merge}\), \(W^{\base}\), and \(W_i\) be the corresponding merged, base, and expert module weights
    \State Read cached \(X_i^{\cali}\) and \(X_i^{\expert}\) from the layer feature snapshot
    \State Construct \(X_i^{\tgt}\), \(G_i\), \(C_i\), and \(\omega_i\) using \cref{eq:featcal-target-feature,eq:featcal-stats,eq:featcal-task-weight}
    \State Form \(W^{\anchor}\) by \cref{eq:featcal-anchor} and store the closed-form weight from \cref{eq:featcal-final-closed-form}
        \If{the linear module has a calibrated bias parameter}
            \State Store the closed-form bias from \cref{eq:featcal-bias-closed-form}
        \EndIf
    \ElsIf{\(t\) is a LayerNorm module}
        \State Read cached normalized current features and store the affine parameters from \cref{eq:featcal-layernorm-closed-form}
    \EndIf
    \EndFor
    \State Load all stored calibrated parameters of this layer into \(M^{\cali}\)
\EndFor
\State \Return \(M^{\cali}\)
\end{algorithmic}
\end{algorithm}

\section{Detailed CLIP Results}
\label{app:clip-detailed-results}

\begin{table}[t]
\centering
\caption{
Full 8-task CLIP-ViT-L/14 results. All numbers are top-1 accuracy (\%).
Parenthesized values in Avg. report changes over the corresponding upstream merger.
}
\label{tab:clip-vit-l14-full}
\setlength{\tabcolsep}{3.2pt}
\resizebox{\textwidth}{!}{
\begin{tabular}{l|cccccccc|c}
\toprule
Method & SUN397 & Cars & RESISC45 & EuroSAT & SVHN & GTSRB & MNIST & DTD & Avg. \\
\midrule
\multicolumn{10}{l}{\emph{Reference methods}} \\
Pre-trained & 68.3 & 77.8 & 71.0 & 58.9 & 58.4 & 50.6 & 76.4 & 55.5 & 64.6 \\
Fine-tuned (STL) & 82.8 & 92.9 & 97.4 & 99.2 & 97.9 & 99.2 & 99.8 & 85.5 & 94.3 \\
Traditional MTL & 79.0 & 89.3 & 94.5 & 98.4 & 96.4 & 98.1 & 99.4 & 83.7 & 92.4 \\
\midrule
\multicolumn{10}{l}{\emph{Model merging and post-merging calibration}} \\
Simple Averaging & 72.5 & 81.5 & 82.3 & 88.5 & 81.6 & 74.0 & 96.6 & 61.8 & 79.9 \\
\quad w/ Surgery & 73.8 & 81.2 & 89.1 & 95.8 & 86.1 & 97.9 & 74.4 & 85.2 & 85.4\posgain{5.5} \\
\quad w/ ProbSurgery & 76.4 & 82.9 & 92.2 & 95.6 & 85.3 & 91.2 & 98.0 & 77.0 & 87.3\posgain{7.3} \\
\rowcolor{FeatCalRow}\quad w/ \method & 77.4 & 88.4 & 91.0 & 97.1 & 95.3 & 93.7 & 97.3 & 73.6 & 89.2\posgain{9.3} \\
\midrule
Task Arithmetic & 72.0 & 79.0 & 80.6 & 84.6 & 87.5 & 83.5 & 98.0 & 58.5 & 80.5 \\
\quad w/ Surgery & 73.4 & 80.5 & 87.9 & 94.6 & 90.6 & 90.3 & 98.7 & 72.1 & 86.0\posgain{5.5} \\
\quad w/ ProbSurgery & 74.3 & 80.8 & 90.9 & 95.4 & 90.7 & 93.0 & 98.8 & 76.6 & 87.6\posgain{6.9} \\
\rowcolor{FeatCalRow}\quad w/ \method & 80.0 & 90.8 & 93.7 & 98.1 & 96.6 & 97.0 & 98.3 & 78.6 & 91.6\posgain{11.1} \\
\midrule
AdaMerging & 78.2 & 90.7 & 90.9 & 96.1 & 94.8 & 97.6 & 98.5 & 81.4 & 91.0 \\
\quad w/ Surgery & 79.0 & 91.0 & 93.9 & 96.9 & 95.4 & 97.8 & 98.8 & 83.1 & 92.0\posgain{1.0} \\
\quad w/ ProbSurgery & 79.4 & 91.6 & 94.3 & 97.5 & 95.6 & 98.1 & 98.9 & 83.4 & 92.4\posgain{1.4} \\
\rowcolor{FeatCalRow}\quad w/ \method & 81.9 & 92.6 & 94.9 & 98.2 & 96.4 & 98.5 & 98.0 & 84.1 & 93.1\posgain{2.1} \\
\midrule
WUDI-Merging & 79.8 & 90.9 & 94.0 & 98.4 & 97.0 & 98.1 & 99.3 & 80.0 & 92.2 \\
\quad w/ Surgery & 80.5 & 91.3 & 95.3 & 98.5 & 97.1 & 98.4 & 99.4 & 81.6 & 92.8\posgain{0.6} \\
\quad w/ ProbSurgery & 80.4 & 91.3 & 95.8 & 98.8 & 97.2 & 98.4 & 99.4 & 82.8 & 93.0\posgain{0.7} \\
\rowcolor{FeatCalRow}\quad w/ \method & 81.5 & 91.7 & 96.5 & 99.0 & 97.8 & 99.2 & 98.6 & 83.6 & 93.5\posgain{1.3} \\
\bottomrule
\end{tabular}
}
\end{table}

\begin{table}[t]
\centering
\caption{
Full 14-task CLIP-ViT-B/32 results. All numbers are top-1 accuracy (\%).
Parenthesized values in Avg. report changes over the corresponding upstream merger.
}
\label{tab:clip-14task-b32-full}
\setlength{\tabcolsep}{2.5pt}
\resizebox{\textwidth}{!}{
\begin{tabular}{l|cccccccccccccc|c}
\toprule
Method & SUN397 & Cars & RESISC45 & EuroSAT & SVHN & GTSRB & MNIST & DTD & Flowers & PCAM & FER & Pets & STL10 & C100 & Avg. \\
\midrule
\multicolumn{16}{l}{\emph{Reference methods}} \\
Pre-trained & 63.2 & 59.6 & 60.3 & 45.0 & 31.6 & 32.5 & 48.3 & 44.2 & 66.5 & 60.6 & 41.2 & 83.3 & 97.1 & 89.8 & 58.8 \\
Fine-tuned (STL) & 74.9 & 78.5 & 95.1 & 99.1 & 97.3 & 98.9 & 99.6 & 79.7 & 88.5 & 98.0 & 71.6 & 92.5 & 97.5 & 88.4 & 90.0 \\
\midrule
\multicolumn{16}{l}{\emph{Model merging and post-merging calibration}} \\
Simple Averaging & 64.8 & 60.4 & 67.1 & 67.0 & 50.7 & 45.6 & 76.6 & 46.9 & 67.4 & 65.2 & 51.6 & 84.2 & 97.2 & 70.4 & 65.4 \\
\quad w/ Surgery & 66.7 & 61.4 & 80.7 & 92.5 & 61.4 & 74.8 & 94.8 & 63.7 & 78.0 & 78.6 & 60.4 & 88.5 & 97.9 & 72.9 & 76.6\posgain{11.2} \\
\quad w/ ProbSurgery & 68.7 & 64.7 & 85.0 & 93.3 & 60.6 & 82.9 & 95.9 & 68.6 & 79.9 & 81.1 & 60.2 & 91.3 & 97.9 & 75.0 & 78.9\posgain{13.5} \\
\rowcolor{FeatCalRow}\quad w/ \method & 68.2 & 67.3 & 80.3 & 92.5 & 86.5 & 77.0 & 95.1 & 60.7 & 75.2 & 86.0 & 63.9 & 89.3 & 97.5 & 74.5 & 79.6\posgain{14.2} \\
\midrule
Task Arithmetic & 62.4 & 55.6 & 64.2 & 67.0 & 53.1 & 85.6 & 48.2 & 55.5 & 61.9 & 76.4 & 54.2 & 82.7 & 95.2 & 64.6 & 66.2 \\
\quad w/ Surgery & 65.6 & 60.5 & 78.7 & 92.7 & 80.6 & 96.3 & 62.7 & 67.5 & 75.4 & 81.2 & 60.5 & 88.0 & 96.8 & 68.4 & 76.8\posgain{10.6} \\
\quad w/ ProbSurgery & 65.9 & 64.0 & 83.2 & 92.6 & 85.6 & 96.2 & 67.1 & 68.7 & 75.3 & 82.2 & 60.6 & 88.2 & 96.7 & 70.5 & 78.3\posgain{12.1} \\
\rowcolor{FeatCalRow}\quad w/ \method & 68.5 & 68.6 & 83.7 & 93.6 & 91.7 & 85.4 & 98.0 & 64.4 & 74.1 & 87.4 & 65.3 & 89.2 & 97.0 & 73.5 & 81.5\posgain{15.3} \\
\midrule
AdaMerging & 66.3 & 70.0 & 82.1 & 92.2 & 86.1 & 89.8 & 98.0 & 61.3 & 73.6 & 51.6 & 65.0 & 87.3 & 96.6 & 68.9 & 77.8 \\
\quad w/ Surgery & 68.5 & 69.5 & 88.5 & 95.2 & 88.8 & 94.6 & 98.4 & 70.9 & 80.8 & 79.7 & 65.3 & 91.4 & 97.6 & 71.6 & 82.9\posgain{5.1} \\
\quad w/ ProbSurgery & 69.0 & 71.1 & 89.8 & 96.0 & 89.2 & 93.9 & 98.8 & 72.6 & 82.0 & 82.7 & 64.5 & 91.4 & 97.6 & 73.3 & 83.7\posgain{5.9} \\
\rowcolor{FeatCalRow}\quad w/ \method & 71.2 & 74.6 & 88.4 & 95.3 & 90.6 & 91.9 & 98.6 & 71.4 & 81.2 & 85.2 & 69.9 & 91.2 & 97.5 & 76.6 & 84.5\posgain{6.7} \\
\midrule
WUDI-Merging & 65.7 & 64.8 & 77.0 & 89.2 & 91.3 & 91.5 & 99.0 & 60.7 & 63.8 & 85.5 & 64.2 & 86.2 & 96.1 & 66.6 & 78.7 \\
\quad w/ Surgery & 66.9 & 68.5 & 84.6 & 96.0 & 92.6 & 95.3 & 99.0 & 67.9 & 76.4 & 86.7 & 64.9 & 88.6 & 96.8 & 70.8 & 82.5\posgain{3.8} \\
\quad w/ ProbSurgery & 67.1 & 69.5 & 86.7 & 96.4 & 92.5 & 95.4 & 99.1 & 68.6 & 76.8 & 85.7 & 63.9 & 89.0 & 96.8 & 70.3 & 82.7\posgain{4.0} \\
\rowcolor{FeatCalRow}\quad w/ \method & 71.8 & 74.6 & 91.2 & 97.1 & 95.9 & 96.4 & 99.3 & 72.8 & 80.5 & 86.8 & 70.6 & 90.7 & 97.9 & 80.7 & 86.2\posgain{7.5} \\
\bottomrule
\end{tabular}
}
\end{table}

\begin{table}[t]
\centering
\caption{
Full 14-task CLIP-ViT-L/14 results. All numbers are top-1 accuracy (\%).
Parenthesized values in Avg. report changes over the corresponding upstream merger.
}
\label{tab:clip-14task-l14-full}
\setlength{\tabcolsep}{2.5pt}
\resizebox{\textwidth}{!}{
\begin{tabular}{l|cccccccccccccc|c}
\toprule
Method & SUN397 & Cars & RESISC45 & EuroSAT & SVHN & GTSRB & MNIST & DTD & Flowers & PCAM & FER & Pets & STL10 & C100 & Avg. \\
\midrule
\multicolumn{16}{l}{\emph{Reference methods}} \\
Pre-trained & 68.2 & 77.9 & 71.3 & 61.2 & 58.4 & 50.5 & 76.3 & 55.5 & 79.3 & 51.2 & 50.0 & 93.2 & 99.4 & 75.1 & 69.1 \\
Fine-tuned (STL) & 82.8 & 92.8 & 97.4 & 91.1 & 97.9 & 99.2 & 99.8 & 85.5 & 97.7 & 91.1 & 75.9 & 95.8 & 99.2 & 93.0 & 92.8 \\
\midrule
\multicolumn{16}{l}{\emph{Model merging and post-merging calibration}} \\
Simple Averaging & 71.2 & 79.0 & 78.7 & 80.4 & 71.3 & 64.6 & 94.3 & 58.7 & 81.9 & 74.2 & 54.8 & 94.6 & 99.3 & 82.4 & 77.5 \\
\quad w/ Surgery & 72.0 & 78.7 & 82.6 & 93.4 & 75.4 & 78.6 & 97.3 & 70.4 & 88.7 & 85.4 & 67.2 & 95.3 & 99.5 & 83.2 & 83.4\posgain{5.9} \\
\quad w/ ProbSurgery & 75.7 & 81.8 & 91.0 & 95.3 & 75.7 & 89.9 & 98.2 & 76.8 & 95.3 & 84.8 & 67.9 & 95.6 & 99.6 & 84.7 & 86.6\posgain{9.1} \\
\rowcolor{FeatCalRow}\quad w/ \method & 74.4 & 85.4 & 86.8 & 95.9 & 91.9 & 86.8 & 96.7 & 68.0 & 91.6 & 84.3 & 67.1 & 95.8 & 99.2 & 84.2 & 86.3\posgain{8.8} \\
\midrule
Task Arithmetic & 71.2 & 75.9 & 77.2 & 76.4 & 70.2 & 95.7 & 58.2 & 72.5 & 80.3 & 73.3 & 57.9 & 94.6 & 98.7 & 79.7 & 77.3 \\
\quad w/ Surgery & 72.3 & 78.3 & 81.1 & 93.9 & 83.7 & 97.7 & 70.0 & 79.8 & 88.8 & 85.9 & 67.1 & 95.0 & 99.3 & 80.8 & 83.8\posgain{6.5} \\
\quad w/ ProbSurgery & 74.8 & 79.8 & 90.2 & 94.7 & 89.9 & 98.1 & 74.3 & 80.2 & 94.0 & 87.2 & 66.1 & 94.9 & 99.3 & 81.7 & 86.1\posgain{8.8} \\
\rowcolor{FeatCalRow}\quad w/ \method & 81.1 & 92.0 & 93.0 & 97.2 & 94.2 & 95.8 & 98.2 & 80.9 & 97.5 & 82.9 & 75.2 & 96.2 & 99.1 & 86.4 & 90.7\posgain{13.4} \\
\midrule
AdaMerging & 77.6 & 90.1 & 91.3 & 95.9 & 94.1 & 96.1 & 98.6 & 78.1 & 95.7 & 51.6 & 74.3 & 95.8 & 99.4 & 83.2 & 87.3 \\
\quad w/ Surgery & 78.7 & 91.0 & 93.5 & 96.5 & 94.5 & 96.8 & 98.9 & 81.5 & 97.3 & 82.8 & 74.1 & 96.0 & 99.5 & 84.2 & 90.4\posgain{3.1} \\
\quad w/ ProbSurgery & 79.0 & 91.4 & 94.6 & 97.0 & 94.9 & 97.4 & 99.0 & 82.7 & 98.0 & 83.2 & 74.4 & 96.4 & 99.5 & 85.2 & 90.9\posgain{3.6} \\
\rowcolor{FeatCalRow}\quad w/ \method & 81.1 & 91.8 & 93.0 & 97.1 & 94.2 & 95.7 & 98.2 & 80.9 & 97.3 & 82.8 & 75.1 & 96.2 & 99.2 & 86.2 & 90.6\posgain{3.3} \\
\midrule
WUDI-Merging & 76.7 & 87.6 & 90.4 & 95.4 & 94.8 & 95.7 & 99.2 & 71.4 & 95.4 & 86.7 & 70.7 & 96.2 & 99.1 & 84.5 & 88.8 \\
\quad w/ Surgery & 78.0 & 89.6 & 92.4 & 97.3 & 95.1 & 96.9 & 99.3 & 77.1 & 96.4 & 89.8 & 72.1 & 96.0 & 99.4 & 85.2 & 90.3\posgain{1.5} \\
\quad w/ ProbSurgery & 78.4 & 89.3 & 94.3 & 97.9 & 95.5 & 97.0 & 99.2 & 79.3 & 96.6 & 89.4 & 72.2 & 96.0 & 99.3 & 85.3 & 90.7\posgain{1.9} \\
\rowcolor{FeatCalRow}\quad w/ \method & 80.5 & 91.1 & 95.4 & 98.5 & 97.2 & 98.5 & 99.0 & 80.6 & 97.5 & 83.7 & 75.0 & 95.6 & 99.2 & 89.3 & 91.5\posgain{2.7} \\
\bottomrule
\end{tabular}
}
\end{table}

\begin{table}[t]
\centering
\caption{
Full 20-task CLIP-ViT-B/32 results. All numbers are top-1 accuracy (\%).
Parenthesized values in Avg. report changes over the corresponding upstream merger.
}
\label{tab:clip-20task-b32-full}
\setlength{\tabcolsep}{1.8pt}
\resizebox{\textwidth}{!}{
\begin{tabular}{l|cccccccccccccccccccc|c}
\toprule
Method & SUN397 & Cars & RESISC45 & EuroSAT & SVHN & GTSRB & MNIST & DTD & Flowers & PCAM & FER & Pets & STL10 & C100 & C10 & Food & F-MNIST & EMNIST & KMNIST & Rendered & Avg. \\
\midrule
\multicolumn{22}{l}{\emph{Reference methods}} \\
Pre-trained & 63.2 & 59.6 & 60.3 & 45.0 & 31.6 & 32.5 & 48.3 & 44.2 & 66.5 & 60.6 & 41.2 & 83.3 & 97.1 & 63.7 & 89.8 & 82.4 & 63.0 & 12.0 & 9.9 & 58.6 & 55.6 \\
Fine-tuned (STL) & 74.9 & 78.5 & 95.1 & 99.1 & 97.3 & 98.9 & 99.6 & 79.7 & 88.5 & 98.0 & 71.6 & 92.5 & 97.5 & 88.4 & 97.6 & 88.4 & 94.8 & 95.6 & 98.2 & 71.3 & 90.3 \\
\midrule
\multicolumn{22}{l}{\emph{Model merging and post-merging calibration}} \\
Simple Averaging & 64.2 & 59.6 & 64.8 & 60.9 & 47.3 & 43.1 & 71.8 & 46.4 & 66.5 & 63.9 & 50.2 & 84.1 & 97.0 & 69.8 & 92.7 & 80.4 & 71.3 & 15.0 & 11.5 & 61.8 & 61.1 \\
\quad w/ Surgery & 66.9 & 61.1 & 79.1 & 91.2 & 56.8 & 74.0 & 95.3 & 61.5 & 76.8 & 80.4 & 59.3 & 89.1 & 97.9 & 72.3 & 94.6 & 81.6 & 83.1 & 63.4 & 65.5 & 60.7 & 75.5\posgain{14.4} \\
\quad w/ ProbSurgery & 68.0 & 64.6 & 83.5 & 92.3 & 58.2 & 82.0 & 95.8 & 67.9 & 81.0 & 79.4 & 59.3 & 90.0 & 97.6 & 74.8 & 94.5 & 82.0 & 84.7 & 73.6 & 66.8 & 63.8 & 78.0\posgain{16.9} \\
\rowcolor{FeatCalRow}\quad w/ \method & 67.6 & 65.3 & 78.4 & 92.0 & 81.1 & 71.2 & 91.5 & 58.3 & 73.5 & 82.8 & 62.7 & 88.6 & 97.3 & 73.4 & 94.9 & 83.4 & 83.1 & 24.3 & 35.2 & 68.4 & 73.7\posgain{12.6} \\
\midrule
Task Arithmetic & 62.0 & 53.7 & 60.9 & 58.1 & 48.9 & 79.4 & 46.1 & 48.5 & 61.1 & 73.4 & 51.4 & 82.3 & 94.9 & 64.6 & 91.4 & 71.9 & 73.9 & 17.8 & 12.2 & 59.9 & 60.6 \\
\quad w/ Surgery & 64.9 & 59.8 & 76.7 & 90.2 & 78.2 & 96.1 & 60.6 & 62.9 & 74.4 & 79.8 & 58.9 & 86.6 & 96.9 & 68.9 & 93.2 & 74.8 & 83.3 & 66.7 & 70.4 & 61.0 & 75.2\posgain{14.6} \\
\quad w/ ProbSurgery & 65.8 & 61.8 & 80.1 & 91.8 & 84.1 & 96.7 & 67.2 & 63.5 & 77.2 & 79.9 & 57.7 & 88.3 & 96.6 & 79.7 & 92.8 & 74.7 & 84.8 & 75.7 & 70.4 & 66.0 & 77.7\posgain{17.1} \\
\rowcolor{FeatCalRow}\quad w/ \method & 67.6 & 67.3 & 83.4 & 94.4 & 91.6 & 88.5 & 97.9 & 63.6 & 73.3 & 84.8 & 65.0 & 87.7 & 97.0 & 76.0 & 95.9 & 82.6 & 89.3 & 42.8 & 66.7 & 72.2 & 79.4\posgain{18.8} \\
\midrule
AdaMerging & 66.2 & 67.0 & 80.0 & 92.1 & 79.1 & 87.2 & 90.5 & 59.5 & 72.7 & 53.2 & 65.5 & 85.2 & 96.4 & 69.3 & 90.8 & 79.7 & 70.0 & 15.4 & 10.0 & 50.6 & 69.0 \\
\quad w/ Surgery & 68.9 & 67.1 & 87.3 & 95.4 & 84.2 & 92.7 & 97.0 & 70.0 & 81.5 & 79.6 & 64.0 & 91.4 & 97.9 & 71.8 & 91.9 & 83.0 & 81.4 & 63.2 & 60.0 & 61.1 & 79.5\posgain{10.5} \\
\quad w/ ProbSurgery & 69.3 & 67.8 & 89.2 & 95.1 & 85.6 & 92.6 & 98.0 & 71.8 & 81.2 & 81.7 & 64.7 & 91.8 & 97.9 & 73.7 & 92.9 & 84.3 & 82.9 & 72.0 & 59.5 & 62.1 & 80.7\posgain{11.7} \\
\rowcolor{FeatCalRow}\quad w/ \method & 70.3 & 70.9 & 85.6 & 94.9 & 87.0 & 89.2 & 97.0 & 67.3 & 79.0 & 82.6 & 68.7 & 90.0 & 97.6 & 75.3 & 95.2 & 84.3 & 83.2 & 27.8 & 28.2 & 68.3 & 77.1\posgain{8.1} \\
\midrule
WUDI-Merging & 55.1 & 44.8 & 59.4 & 78.5 & 79.7 & 82.9 & 98.1 & 50.3 & 49.3 & 82.0 & 58.5 & 77.7 & 93.4 & 59.9 & 90.3 & 53.1 & 83.9 & 35.4 & 40.0 & 69.0 & 67.1 \\
\quad w/ Surgery & 57.5 & 60.1 & 73.8 & 93.4 & 84.9 & 89.5 & 98.3 & 59.3 & 67.7 & 82.2 & 61.0 & 82.1 & 95.2 & 64.8 & 90.8 & 57.9 & 87.7 & 75.3 & 78.4 & 70.4 & 76.5\posgain{9.4} \\
\quad w/ ProbSurgery & 57.8 & 60.7 & 76.7 & 94.3 & 86.2 & 92.7 & 98.4 & 63.0 & 69.3 & 82.3 & 60.6 & 83.9 & 95.0 & 65.4 & 91.9 & 57.1 & 88.6 & 80.6 & 77.4 & 70.1 & 77.6\posgain{10.5} \\
\rowcolor{FeatCalRow}\quad w/ \method & 70.5 & 71.6 & 89.1 & 96.6 & 93.9 & 93.5 & 98.4 & 69.1 & 77.9 & 85.4 & 68.7 & 90.0 & 97.7 & 79.5 & 96.6 & 85.1 & 91.4 & 58.1 & 85.9 & 71.1 & 83.5\posgain{16.4} \\
\bottomrule
\end{tabular}
}
\end{table}

\begin{table}[t]
\centering
\caption{
Full 20-task CLIP-ViT-L/14 results. All numbers are top-1 accuracy (\%).
Parenthesized values in Avg. report changes over the corresponding upstream merger.
}
\label{tab:clip-20task-l14-full}
\setlength{\tabcolsep}{1.8pt}
\resizebox{\textwidth}{!}{
\begin{tabular}{l|cccccccccccccccccccc|c}
\toprule
Method & SUN397 & Cars & RESISC45 & EuroSAT & SVHN & GTSRB & MNIST & DTD & Flowers & PCAM & FER & Pets & STL10 & C100 & C10 & Food & F-MNIST & EMNIST & KMNIST & Rendered & Avg. \\
\midrule
\multicolumn{22}{l}{\emph{Reference methods}} \\
Pre-trained & 68.2 & 77.9 & 71.3 & 61.2 & 58.4 & 50.5 & 76.3 & 55.5 & 79.3 & 51.2 & 50.0 & 93.2 & 99.4 & 75.1 & 95.6 & 91.2 & 67.0 & 12.3 & 9.7 & 68.9 & 65.6 \\
Fine-tuned (STL) & 82.8 & 92.8 & 97.4 & 91.1 & 97.9 & 99.2 & 99.8 & 85.5 & 97.7 & 91.1 & 75.9 & 95.8 & 99.2 & 93.0 & 99.1 & 94.7 & 95.3 & 95.4 & 98.3 & 80.5 & 93.1 \\
\midrule
\multicolumn{22}{l}{\emph{Model merging and post-merging calibration}} \\
Simple Averaging & 70.7 & 77.7 & 76.4 & 75.3 & 69.5 & 62.1 & 93.7 & 57.7 & 80.8 & 73.6 & 52.7 & 94.2 & 99.2 & 81.7 & 97.0 & 90.7 & 77.4 & 16.1 & 10.4 & 66.1 & 71.2 \\
\quad w/ Surgery & 72.0 & 78.1 & 82.1 & 95.0 & 75.4 & 79.2 & 97.5 & 70.9 & 89.0 & 84.9 & 66.5 & 94.8 & 99.3 & 83.5 & 98.1 & 93.0 & 86.3 & 60.5 & 67.6 & 70.8 & 82.2\posgain{11.0} \\
\quad w/ ProbSurgery & 75.1 & 78.7 & 90.7 & 94.5 & 74.7 & 89.5 & 97.9 & 75.1 & 94.2 & 84.7 & 67.2 & 94.6 & 99.5 & 84.0 & 98.0 & 93.6 & 87.2 & 78.4 & 73.0 & 73.8 & 85.2\posgain{14.1} \\
\rowcolor{FeatCalRow}\quad w/ \method & 73.4 & 83.1 & 84.8 & 95.6 & 89.7 & 82.2 & 97.2 & 64.9 & 89.8 & 82.9 & 65.7 & 95.8 & 99.1 & 83.1 & 98.0 & 91.6 & 87.4 & 36.4 & 14.4 & 71.9 & 79.3\posgain{8.1} \\
\midrule
Task Arithmetic & 70.4 & 74.1 & 73.9 & 66.3 & 65.6 & 95.1 & 56.6 & 69.9 & 78.6 & 70.4 & 55.7 & 94.2 & 98.6 & 79.1 & 96.6 & 87.6 & 80.8 & 17.6 & 10.6 & 63.6 & 70.3 \\
\quad w/ Surgery & 71.7 & 76.5 & 79.6 & 94.3 & 80.6 & 98.1 & 69.3 & 78.5 & 88.6 & 85.8 & 66.2 & 94.6 & 99.0 & 80.5 & 97.5 & 90.1 & 86.4 & 66.2 & 74.3 & 71.5 & 82.4\posgain{12.1} \\
\quad w/ ProbSurgery & 76.4 & 86.6 & 89.2 & 96.1 & 92.9 & 91.0 & 98.5 & 68.5 & 93.7 & 82.9 & 70.0 & 95.9 & 98.9 & 84.3 & 98.1 & 91.3 & 90.4 & 50.9 & 37.6 & 73.2 & 83.3\posgain{13.0} \\
\rowcolor{FeatCalRow}\quad w/ \method & 77.2 & 86.9 & 90.0 & 96.4 & 93.8 & 93.0 & 98.6 & 69.1 & 94.9 & 82.7 & 70.4 & 95.8 & 99.0 & 84.9 & 98.1 & 92.5 & 92.0 & 54.0 & 51.2 & 73.2 & 84.7\posgain{14.4} \\
\midrule
AdaMerging & 76.8 & 89.6 & 90.3 & 95.1 & 91.9 & 95.2 & 98.5 & 76.4 & 94.5 & 52.0 & 69.9 & 95.5 & 99.3 & 82.3 & 96.5 & 90.5 & 85.2 & 14.0 & 10.0 & 81.7 & 79.3 \\
\quad w/ Surgery & 78.3 & 90.2 & 92.6 & 96.7 & 92.7 & 96.7 & 98.6 & 80.8 & 96.9 & 84.3 & 72.2 & 95.9 & 99.3 & 83.7 & 96.8 & 92.0 & 87.3 & 55.8 & 61.7 & 77.7 & 86.5\posgain{7.2} \\
\quad w/ ProbSurgery & 78.4 & 90.3 & 94.1 & 96.7 & 94.1 & 96.6 & 98.8 & 81.3 & 97.3 & 83.3 & 72.2 & 95.9 & 99.4 & 84.6 & 97.1 & 93.0 & 88.7 & 73.9 & 65.1 & 81.7 & 88.1\posgain{8.8} \\
\rowcolor{FeatCalRow}\quad w/ \method & 80.3 & 91.1 & 91.7 & 96.7 & 91.7 & 93.3 & 98.7 & 78.7 & 96.9 & 82.4 & 73.0 & 96.1 & 99.0 & 84.8 & 97.9 & 91.6 & 89.0 & 36.7 & 10.3 & 81.8 & 83.1\posgain{3.8} \\
\midrule
WUDI-Merging & 70.3 & 72.1 & 73.3 & 69.7 & 81.8 & 84.3 & 98.1 & 56.8 & 85.9 & 83.7 & 64.2 & 94.5 & 97.5 & 72.9 & 95.7 & 83.5 & 89.6 & 33.2 & 33.3 & 74.8 & 75.8 \\
\quad w/ Surgery & 71.7 & 80.5 & 81.7 & 95.3 & 86.4 & 87.7 & 98.7 & 69.2 & 91.3 & 88.6 & 66.2 & 94.1 & 98.2 & 75.1 & 95.7 & 85.1 & 90.7 & 74.4 & 88.7 & 79.0 & 84.9\posgain{9.1} \\
\quad w/ ProbSurgery & 72.2 & 80.8 & 88.3 & 95.2 & 86.8 & 92.8 & 98.9 & 73.2 & 93.1 & 88.4 & 66.2 & 93.7 & 98.3 & 76.0 & 95.8 & 85.1 & 91.0 & 85.5 & 89.6 & 79.7 & 86.5\posgain{10.7} \\
\rowcolor{FeatCalRow}\quad w/ \method & 79.4 & 89.4 & 94.3 & 98.0 & 96.1 & 97.2 & 98.9 & 77.2 & 97.2 & 83.0 & 74.4 & 95.7 & 99.2 & 88.4 & 98.7 & 93.8 & 93.8 & 73.8 & 88.7 & 79.3 & 89.8\posgain{14.0} \\
\bottomrule
\end{tabular}
}
\end{table}

\clearpage

\section{Additional Diagnostic: Comparison with ProbSurgery}
\label{app:probsurgery-diagnostic}

\begin{figure}[t]
\centering
\includegraphics[width=\textwidth]{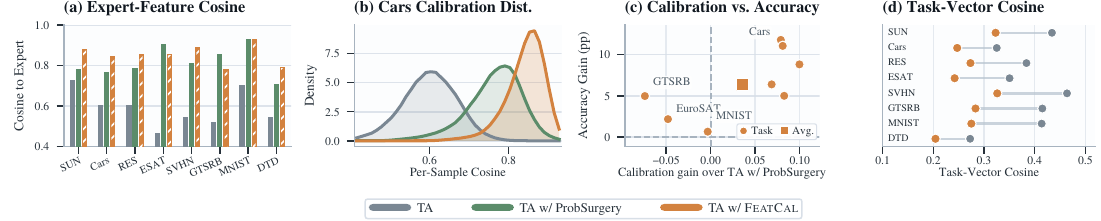}
\caption{
Feature-calibration diagnostics for TA w/ \method{} versus TA w/ ProbSurgery on CLIP-ViT-B/32 Task Arithmetic. 
(a) Task-wise mean cosine between final projected features and same-task expert features.
(b) Per-sample expert-feature cosine on Stanford Cars.
(c) Expert-feature cosine gain over TA w/ ProbSurgery versus top-1 accuracy gain.
(d) Backbone task-vector cosine to same-task expert vectors for TA and TA w/ \method{}.
}
\label{fig:probsurgery-feature-calibration-diagnostics}
\end{figure}

\Cref{fig:probsurgery-feature-calibration-diagnostics} repeats the diagnostic comparison with the stronger TA w/ ProbSurgery baseline.
Panels (a) and (b) show the same average feature-space pattern as in the main text: TA w/ \method{} raises the macro-average expert cosine from 0.814 to 0.850 over TA w/ ProbSurgery, improves 8-task average accuracy from 79.1 to 85.5, and increases the Stanford Cars mean expert cosine by 0.079.

Panel (c) again shows that expert-feature cosine is useful but not a complete accuracy explanation.
TA w/ \method{} improves accuracy over TA w/ ProbSurgery on every task, yet its final-feature expert cosine is lower on EuroSAT, GTSRB, and MNIST; the largest negative cosine change is on GTSRB (\(-0.074\)), where accuracy still improves by \(+4.97\) points.
Panel (d) shows the same feature-space and parameter-space mismatch as \cref{fig:feature-calibration-diagnostics}: although TA w/ \method{} improves final-feature alignment on average, the calibrated backbone has lower same-task expert task-vector cosine than the TA baseline on every task in this run.
Thus, the ProbSurgery comparison reinforces the main diagnostic message: feature calibration can improve the deployed final-feature distribution without producing closer task-vector alignment in parameter space, and neither diagnostic alone determines task accuracy.

\section{Efficiency Experiment Protocol}
\label{app:efficiency-protocol}

\paragraph{Task and model setting.}
The efficiency analysis in \cref{sec:analysis} uses the FusionBench CLIP-ViT-B/32 8-task TA setting with Task Arithmetic as the shared upstream merger.
The Task Arithmetic scaling factor is \(0.3\), and all post-merging methods start from the same saved TA model.
We evaluate Surgery, ProbSurgery, and \method on the 8 image-classification tasks listed in \cref{sec:clip-results}.
We follow each method's native FusionBench calibration interface: Surgery and ProbSurgery use their test-time adaptation calibration stream, while \method uses its calibration split.
Final accuracy is always measured by a separate 8-task TA evaluation pass after all calibration jobs finish, so this analysis should be read as an implementation-level post-TA efficiency comparison rather than a controlled split-ablation.

\paragraph{Data budgets and seeds.}
For each method, \(n\) denotes the exact number of calibration examples per task.
We use \(n\in\{8,16,32,64,128,256,512\}\) and report mean and standard deviation over 3 seeds, \(42\), \(43\), and \(44\).
Surgery and ProbSurgery set \texttt{max\_samples\_per\_task}=n; \method sets \texttt{num\_regmean\_examples}=n and truncates the final dataloader batch so that exactly \(n\) examples per task are used.

\begin{table}[t]
\centering
\caption{
Full sample-budget accuracy values corresponding to \cref{fig:ta8-data-efficiency}.
Entries are mean top-1 accuracy (\%) in the 8-task TA setting and standard deviation over 3 seeds.
}
\label{tab:ta8-full-data-efficiency}
\setlength{\tabcolsep}{3.6pt}
\resizebox{\textwidth}{!}{
\begin{tabular}{l|ccccccc}
\toprule
Method
& \(n=8\)
& \(n=16\)
& \(n=32\)
& \(n=64\)
& \(n=128\)
& \(n=256\)
& \(n=512\) \\
\midrule
TA w/ Surgery & 66.2 $\pm$ 1.1 & 66.7 $\pm$ 1.2 & 69.9 $\pm$ 0.3 & 72.6 $\pm$ 0.2 & 74.9 $\pm$ 0.2 & 76.8 $\pm$ 0.2 & 77.8 $\pm$ 0.3 \\
TA w/ ProbSurgery & 63.4 $\pm$ 1.4 & 66.1 $\pm$ 1.3 & 69.5 $\pm$ 0.6 & 72.4 $\pm$ 0.5 & 75.7 $\pm$ 0.5 & 78.6 $\pm$ 0.2 & 80.7 $\pm$ 0.2 \\
\rowcolor{FeatCalRow} TA w/ \method & \textbf{82.9 $\pm$ 0.3} & \textbf{83.8 $\pm$ 0.2} & \textbf{84.4 $\pm$ 0.0} & \textbf{84.9 $\pm$ 0.1} & \textbf{85.2 $\pm$ 0.1} & \textbf{85.4 $\pm$ 0.1} & \textbf{85.5 $\pm$ 0.0} \\
\bottomrule
\end{tabular}
}
\end{table}

\paragraph{Optimization and dataloading.}
All 3 post-merging methods use calibration dataloader batch size \(16\), matching the original FusionBench AdaMerging/Surgery dataloader default, and the profiling run uses 8 dataloader workers.
Surgery and ProbSurgery use learning rate \(10^{-3}\).
Surgery and ProbSurgery run for \(1000\) calibration steps, and their internal validation/evaluation hooks are disabled for the main efficiency comparison.
This avoids mixing calibration cost with intermediate model-selection cost.
\method follows the same 8-task TA calibration protocol as the corresponding main result in \cref{sec:exp-setup}.

\paragraph{Measurement protocol.}
The profiler separates calibration and evaluation.
It first generates or reuses the TA baseline model, then runs all post-merging calibration jobs, and only after all calibration jobs finish evaluates the TA baseline and all calibrated models on the 8-task TA task pool.
Calibration efficiency is therefore reported using calibration-only wall-clock time, GPU energy, and CPU RSS columns, while final accuracy is reported from the separate evaluation stage.
CPU RSS is the peak resident host memory of the profiled process tree.
The profiling run also evaluates saved Surgery and ProbSurgery models at steps \(200,400,600,800,1000\) as an offline model-selection audit.
Those model-evaluation passes are excluded from the runtime columns in \cref{tab:ta8-runtime-efficiency-n256,tab:ta8-overall-runtime-appendix}; if model search is reported, it is labeled as an offline model-search variant rather than the default final-model result.

\begin{table}[t]
\centering
\caption{
Cross-budget calibration-stage resource summary.
Values average over all data budgets and 3 seeds per budget; GPU energy is computed per run from average GPU power draw and calibration-stage time.
Final 8-task TA evaluation and offline model search are excluded.
}
\label{tab:ta8-overall-runtime-appendix}
\setlength{\tabcolsep}{4pt}
\begin{tabularx}{0.82\linewidth}{@{}l>{\raggedleft\arraybackslash}X>{\raggedleft\arraybackslash}X>{\raggedleft\arraybackslash}X@{}}
\toprule
Method
& \shortstack{Avg. stage\\time (s)}
& \shortstack{Avg. GPU\\energy (Wh)}
& \shortstack{Avg. peak\\CPU RSS (GiB)} \\
\midrule
TA w/ Surgery & 209.9 & 17.0 & 102.9 \\
TA w/ ProbSurgery & 220.3 & 17.3 & 95.1 \\
\rowcolor{FeatCalRow} TA w/ \method & \textbf{52.1} & \textbf{1.5} & \textbf{20.4} \\
\bottomrule
\end{tabularx}
\end{table}

\paragraph{Shared hardware and software.}
All experiments in this paper use the same single-GPU worker class, including the CLIP main and extended results, the FLAN-T5 results, the feature-calibration diagnostics, the corrupted-calibration study, the sensitivity study, and the efficiency study.
Each run uses a single NVIDIA A800-SXM4-80GB GPU with driver \(550.163.01\), CUDA \(12.4\) as reported by \texttt{nvidia-smi}, an Intel Xeon Platinum 8358 CPU node with 128 logical CPUs, Python \(3.12.2\), and PyTorch \(2.10.0+\mathrm{cu}128\).
No experiment uses distributed training or multi-GPU parallelism.
The profiler records calibration-stage wall-clock time, GPU energy, and peak CPU RSS for each sample-budget run in the efficiency study; these measurements are summarized in \cref{tab:ta8-runtime-efficiency-n256,tab:ta8-overall-runtime-appendix}.
The logs also record the launch command, environment variables, \texttt{nvidia-smi}, CPU information, package versions, and output CSV files used to populate \cref{fig:ta8-data-efficiency,tab:ta8-runtime-efficiency-n256,tab:ta8-full-data-efficiency,tab:ta8-overall-runtime-appendix}.

\section{Corrupted-Calibration Robustness Protocol}
\label{app:corrupted-calibration-protocol}

The corrupted-calibration analysis in \cref{sec:analysis} uses the FusionBench CLIP-ViT-B/32 8-task TA setting with Task Arithmetic as the shared upstream merger.
All methods start from the same Task Arithmetic model, and final accuracy is measured on the clean test sets from this 8-task TA setting.
Only the data consumed by the post-merging calibration stage are corrupted.
For \method, this is the \(256\)-image-per-task calibration set; for Surgery and ProbSurgery, this is each method's native calibration stream under the same per-task budget.

We apply a single ImageNet-C-style severity-3 corruption at a time: Gaussian noise, motion blur, or fog.
The uncalibrated Task Arithmetic baseline has no calibration stream, so it is reported only in the clean column.
All calibrated rows in \cref{tab:ta8-corrupted-calibration} report the final calibrated models from the corresponding corrupted-data run.

\section{Additional MergeBench Expert Results}
\label{app:mergebench-results}

We report the individual MergeBench expert references omitted from the compact main-text table.
All entries in \cref{tab:mergebench-expert-results} are percentages, and the average is the mean over MATH-500, GSM8K, HumanEval+, MBPP+, IFEval, and ARC-Challenge.
HumanEval+ and MBPP+ report pass@1.

\begin{table}[t]
\centering
\caption{
Individual MergeBench expert results on Llama-family models.
}
\label{tab:mergebench-expert-results}
\setlength{\tabcolsep}{3.8pt}
\resizebox{\textwidth}{!}{
\begin{tabular}{ll|cccccc|c}
\toprule
Series / Model & Method & MATH-500 & GSM8K & HumanEval+@1 & MBPP+@1 & IFEval & ARC-C & Avg. \\
\midrule
\multirow{5}{*}{Llama-3.2-3B-Instruct}
& Base & 47.8 & 72.6 & 49.4 & 56.6 & 67.3 & 73.1 & 61.1 \\
& Coding Expert & 27.4 & 55.5 & 54.9 & 57.1 & 49.5 & 71.5 & 52.7 \\
& Math Expert & 46.4 & 80.6 & 48.2 & 55.3 & 53.0 & 71.1 & 59.1 \\
& Instruction Expert & 49.0 & 73.6 & 54.3 & 56.1 & 69.7 & 72.7 & 62.6 \\
& Fine-tuned (STL) & 49.0 & 80.6 & 54.9 & 57.1 & 69.7 & 73.1 & 64.1 \\
\midrule
\multirow{5}{*}{Llama-3.1-8B-Instruct}
& Base & 48.6 & 83.5 & 62.2 & 63.0 & 72.8 & 80.9 & 68.5 \\
& Coding Expert & 21.2 & 53.4 & 66.5 & 59.0 & 43.8 & 76.5 & 53.4 \\
& Math Expert & 47.4 & 83.9 & 22.0 & 36.5 & 15.5 & 47.4 & 42.1 \\
& Instruction Expert & 52.6 & 85.3 & 67.1 & 63.8 & 72.5 & 78.8 & 70.0 \\
& Fine-tuned (STL) & 52.6 & 85.3 & 67.1 & 63.8 & 72.8 & 80.9 & 70.4 \\
\bottomrule
\end{tabular}
}
\end{table}

\section{Broader Context and Future Directions}
\label{app:broader-outlook}

Our experiments focus on post-merging calibration after a base merger has already produced a single deployable model, but this setting is part of a wider model-fusion agenda.
Recent surveys frame model fusion as a route toward scalable, sustainable, and more broadly accessible AI systems~\citep{zhou2026model,zhou2025democratizing}.
For \method, an immediate next step is to understand how closed-form calibration scales with model size, task count, and available compute.
Scaling-law and parameter-management studies for large-language-model merging provide useful tools for this question, especially when calibration must fit into a constrained memory or budget envelope~\citep{wang2025model,wang2026mergepipe}.

Another direction is to extend feature calibration beyond accuracy-oriented benchmarks.
Preference-oriented fusion and distillation suggest that merged models can be shaped by preference signals rather than only task labels or expert features~\citep{gu2025infifpo,gu-etal-2025-capturing}, while logits-level distillation gives a complementary way to transfer relational information between models~\citep{wang2025infigfusion}.
As calibrated merged models move into interactive and multimodal settings, evaluation should also account for response uncertainty and possible misleading contexts~\citep{dang-etal-2025-exploring}.
Finally, systems support remains important: decentralized evaluation pipelines and low-precision reasoning-model training recipes point to deployment settings where calibrated model fusion should be efficient, auditable, and inexpensive to run~\citep{yang2026inficoevalchain,wang2025infir2comprehensivefp8training}.
These directions are complementary to the layer-wise calibration mechanism studied here and ask how post-merging feature calibration should be scaled, evaluated, and deployed.

\end{document}